%% file: arxiv_ICLR_bayesian_coreset.tex
\documentclass{article} % For LaTeX2e
\usepackage{iclr2024_conference,times}
\usepackage{booktabs}
\usepackage{xpatch}

\usepackage{placeins}
\usepackage{hyperref}
\usepackage{caption, array}
\usepackage{natbib}  % DO NOT CHANGE THIS AND DO NOT ADD ANY OPTIONS TO IT
\usepackage{tabularx}
\usepackage{thmtools}

\usepackage{makecell}

\usepackage{graphicx}
\usepackage{amsmath,amsthm,amssymb}
\usepackage{subcaption}
\usepackage{tablefootnote}
\usepackage{algorithm}
\usepackage{algorithmic}
\usepackage{hyperref}

\input{macros}

\title{Bayesian Coreset Optimization for \\ Personalized Federated Learning}

\author{Prateek Chanda \\
Department of Computer Science \\
Indian Institute of Technology \\
Bombay, India \\
\texttt{prateekch@cse.iitb.ac.in} \\
\And
Shrey Modi \\
Department of Computer Science\\
Indian Institute of Technology\\
Bombay, India \\
\texttt{200020135@iitb.ac.in} \\
\And
Ganesh Ramakrishnan \\
Department of Computer Science\\
Indian Institute of Technology\\
Bombay, India \\
\texttt{ganesh@cse.iitb.ac.in}
}

\begin{document}

\maketitle

\begin{abstract}
In a distributed machine learning setting like Federated Learning where there are multiple clients involved which update their individual weights to a single central server, often training on the entire individual client's dataset for each client becomes cumbersome. To address this issue we propose $\methodprop$: a personalized coreset weighted federated learning setup where the training updates for each individual clients are forwarded to the central server based on only individual client coreset based representative data points instead of the entire client data. Through theoretical analysis we present how the average generalization error is minimax optimal up to logarithm bounds (upper bounded by $\mathcal{O}(n_k^{-\frac{2 \beta}{2 \beta+\boldsymbol{\Lambda}}} \log ^{2 \delta^{\prime}}(n_k))$) and lower bounds of $\mathcal{O}(n_k^{-\frac{2 \beta}{2 \beta+\boldsymbol{\Lambda}}})$, and how the overall generalization error on the data likelihood differs from a vanilla Federated Learning setup as a closed form function ${\boldsymbol{\Im}}(\boldsymbol{w}, n_k)$ of the coreset weights $\boldsymbol{w}$ and coreset sample size $n_k$. 
Our experiments on different benchmark datasets based on a variety of recent personalized federated learning architectures show significant gains as compared to random sampling on the training data followed by federated learning, thereby indicating how intelligently selecting such training samples can help in performance. Additionally, through experiments on medical datasets our proposed method showcases some gains as compared to  other submodular optimization based approaches used for subset selection on client's data.
\end{abstract}

\section{Introduction}

Distributed machine learning has a wide variety of applications in various domains like in recommendation systems \citep{zhang2022pipattack}, healthcare and other areas. Given the advantages of data-privacy, heterogeneity and resource efficiency, federated learning in particular stands out among all other learning methods. 
For instance in recommendation systems based work \citep{luo2022personalized}  \citep{zhang2023dual} or in healthcare \citep{lu2022personalized}, personalization and privacy preserving based recommendations are currently most focused upon. However, often these require inference on larger datasets which is often computationally expensive given the fact that often these datasets only consist of a subset of representative data points \citep{kaushal2018learning}, \citep{maheshwari2020semi}.

 Recently, researchers have tried to use Bayesian coresets for performing inference \citep{campbell2019sparse},\citep{campbell2018bayesian}. Bayesian coreset refers to a subset of the complete dataset which can be used to approximate the posterior inference as if it was performed on the complete dataset. In Federated Learning, we have a distributed machine learning model where there are a set of clients each having their own share of data and a server. This is prevalent in many production level systems such as recommendation systems or machine learning models deployed for mobile apps. It is imperative for each client to get the same user satisfiability and personalized experience with only a small chunk of data on their local devices. Instead of training on the entire user data for a particular client, what if we can only learn based on a representative set of the user data for each client and achieve near optimal accuracy? In this space \citep{huang2022coresets} perform coresets optimization for vertical federated learning where empirically they show that coresets optimization help reducing the communication complexity among clients and server in the vertical FL setting.

Our contributions are as follows: 
% TODO :\pc{contributions need to be more succinct and to the point}

\begin{itemize}
    \item Proposal of a new architecture to incorporate bayesian coresets optimization in the space of Federated Learning.
    \item Proposal of multiple novel objective functions that take into account the optimization problem of general Federated Learning  in a Bayesian coresets setting with particular focus on personalized coreset weights for each individual clients.
    \item Theoretical analysis on the convergence rate shows our approach $\methodprop$ achieves convergence rate within logarithmic bounds.
    \item Experimental results across several benchmark datasets conducted with a wide array of pre-existing baselines show promising results towards good performance in terms of model accuracy even with less data at each client's end.

   % We share our code on \href{https://github.com/prateekiiest/BayesianCoreset-FederatedLearning}{Github}.
\end{itemize}
% \fancyfoot[R]{}

\section{Related Work}

% TODO: \pc{categorise model centric vs data centric}

It is well known in literature that training datasets offer diminishing returns in terms of performance. It has also been demonstrated that one can train and obtain better returns in performance and energy efficiency by training models over subsets of data that are meticulously selected~\citep{datashapley, datavaluation, datavalue, strubell2019energy}. This leads us to the problem of coreset selection that deals with approximating a desirable quantity ({\em e.g.}, gradient of a loss function) over an entire dataset with a weighted subset of it. 
% The approximation error is contingent on the quality of samples selected in the coreset.
Traditional methods of coreset selection have used a submodular proxy function to select the coreset and are model dependent~\citep{submodcoreset1, submodcoreset2, submodcoreset3, coresetkmeans, submodcoreset4}. Coreset selection with deep learning models has become popular in recent years~\citep{craig, crust, gradmatch, deepcoreset1, facilitylocation}. Coreset selection to approximately match the full-batch training gradient is proposed in~\citep{craig}. %Killamsetty {\em et. al.}~
\citet{gradmatch} propose algorithms to select coresets that either match the full-batch gradient or the validation gradient. %Mirzasoleiman {\em et. al.}~
\citet{crust} propose an approach to select a coreset that admits a low-rank jacobian matrix and show that such an approach is robust to label noise. 
Most existing coreset selection approaches are proposed in conventional settings wherein all the data is available in one place. In FL, since no client or the server gets a holistic view of the training dataset, coresets can at best approximate only local data characteristics and are thus inherently sub-optimal.

% TODO \pc{reduce related work and keep important ones}
Coreset selection in federated learning has remained largely unexplored because of the intricacies involved due to privacy and data partition across clients. Federated Learning can be modelled as a cooperative game where it often uses Shapley values to select clients whose updates result in the best reduction of the loss on a validation dataset held by the server. One work that comes very close to ours is that of~\citep{diverseblimes}, that selects a coreset of clients\footnote{As against selecting a coreset of data instances} whose update represents the update aggregated across all the clients. They apply facility location algorithm on the gradient space to select such a coreset of clients. In contrast to all these approaches, our proposal attempts to select a coreset of the dataset at each client and is thus more fine-grained than such prior works. One another paradigm of FL that is in contrast to our setup is Personalized FL, where the aim is to train specialised models for each individual client ~\citep{collins2021exploiting, FL_MAML,pmlr-v162-marfoq22a,Li_2022_CVPR, NEURIPS2021_f8580959}. While personalized FL focuses on finetuning the model to match each client's data distribution, we build models to account for just the server's distribution, similar to ~\citep{scaffold}. 

The algorithms that ensure privacy in FL include differential privacy~\citep{diff_privacy_1, diff_privacy_2}, Homomorphic encryption ~\citep{homomorphic_1, homomorphic_2}, {\em etc.}; which is not the focus of our work. However, we note that all these methods can be easily integrated with our solution approach and can be used in conjunction.

%\subsubsection{
%Bayesian Inference with Subsampling Strategies}

%Several notable methods in Bayesian inference that incorporate subsampling strategies have been extensively explored in the literature, including seminal works such as \citep{jaggi2013revisiting} and \citep{hoffman2013stochastic}.

%Previous works on Personalized Federated Learning has primarily focused on 
%\red{Previous work on Personalized Federated ML} 

%\red{Previous work on Personalized Federated ML with coreset analysis / subset selection}

%We discuss here subset selection based federated learning strategies
%\subsubsection{Submodular maximization based Federated Learning}

%\citep{balakrishnan2022diverse} considers a diversity based subset selection of clients that can send updates back to the server.

\begin{figure}
    \centering
\def\svgwidth{\linewidth}
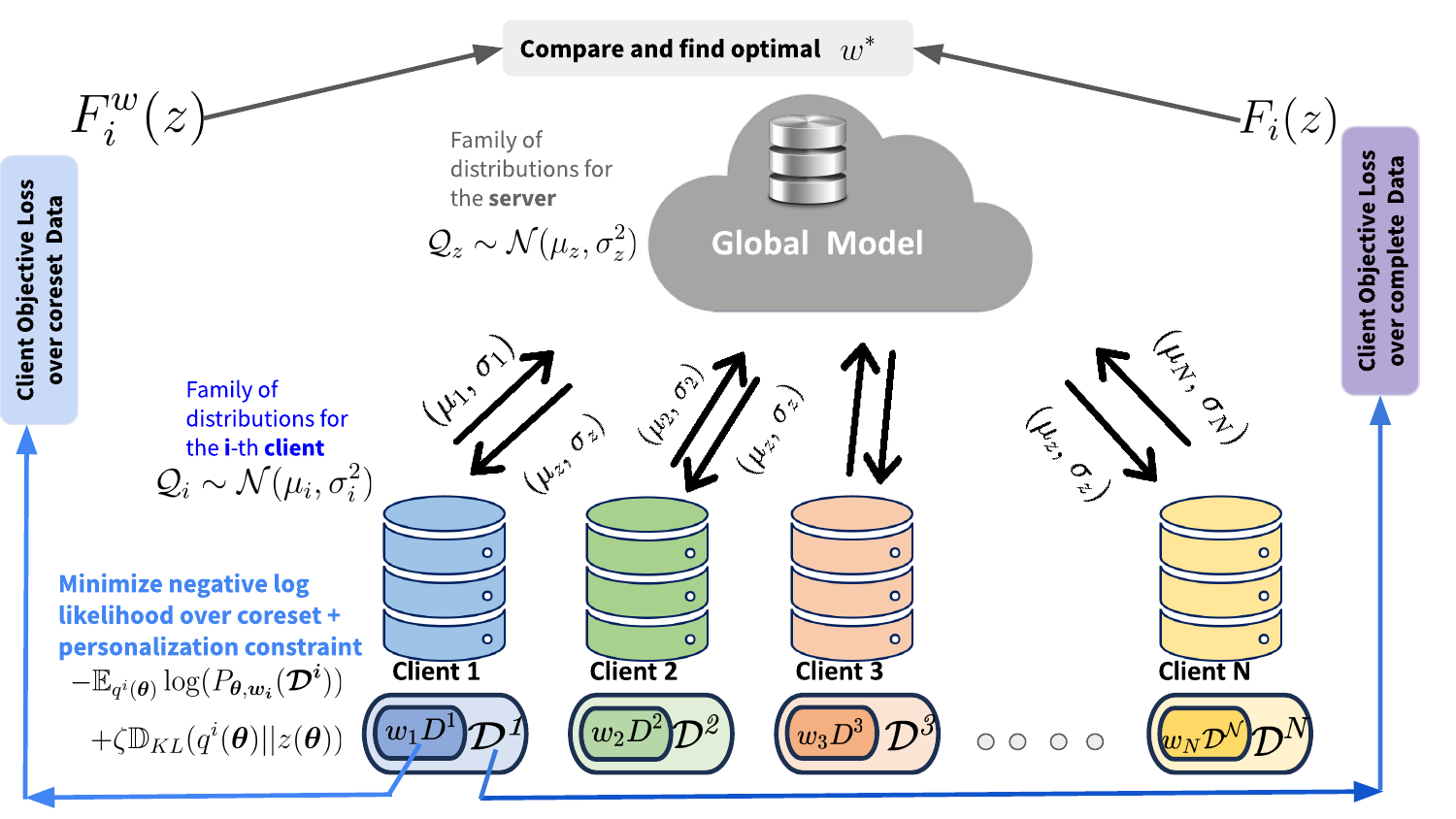
    \caption{\textbf{System Diagram} : Coreset Weighted Personalized Federated Learning model with parameters under Gaussian assumptions. Each client uploads its updated distribution to the server based its corresponding coreset training data (each client $i$'s data $\boldsymbol{\mathcal{D}^i}$ is weighted by $\boldsymbol{w}_i$) and then the aggregrated global distribution is utilised from the server.}
    \label{fig:sysarch}
\end{figure}

\section{Preliminaries and Problem Setting}\label{Sec:Notations}

%We start by introducing some notations as listed below.

%\begin{table}[h!]
%\begin{tabular}{ll}\toprule
%\parbox[t]{3cm}{\textbf{Symbols}} &  \parbox[t]{5cm}{\textbf{Definition}} \\ \toprule
%$\boldsymbol{w}$   & Coreset weights used in Bayesian Coreset Optimization problem \\
%\midrule
%$\boldsymbol{\theta}$ & Network parameters 
%\\
%\midrule
%$\pi(\boldsymbol{\theta})$ & prior distribution over client data $\boldsymbol{\mathcal{D}}$\\
%\midrule
%$P_{\boldsymbol{\theta}}(\boldsymbol{\mathcal{D}})$ & True log-likelihood over the client data $\boldsymbol{\mathcal{D}}$. \\
%\midrule
%$P_{\boldsymbol{\theta},w}(\boldsymbol{\mathcal{D}})$ & Log-likelihood over the weighted coreset $\boldsymbol{w}$ over client data $\boldsymbol{\mathcal{D}}$\\
%\midrule
%$\zeta$ & degree of personalization for each client w.r.t global model\\
%\midrule
%$\mathcal{Q}_z$ & denote the family of distributions for global parameter (server).\\
%\midrule
%$\mathcal{Q}_{i}$ & denotes the family of distributions for the $i$-the client respectively.\\
%\bottomrule 
%\end{tabular}
%\caption{Notations}
%\end{table}

We consider the problem setting from \citep{zhang2022personalized} as follows: Consider  a distributed system that includes one server and $N$ clients. Let the $i$th client's dataset be $\boldsymbol{\mathcal{D}}^i = \{(\boldsymbol{x}^i_j, \boldsymbol{y}^i_j)\}_{j=1}^n$.(assuming all $N$ clients have the same sample size $n$)  Further, let the $i$-th client satisfy a regression model with random covariates as
$\boldsymbol{y}_j^i = f^i(\boldsymbol{x}^i_j) + \epsilon^i_j , \hspace{6pt}
    \forall j \in [1, \dots n]$
where $\boldsymbol{x}_j^i \in \mathbb{R}^{s_0}, \boldsymbol{y}_j^i \in \mathbb{R}^{s_{L+1}}$ for $\forall i=$ $1, \ldots, N$. Here $f^i(\bullet): \mathbb{R}^{s_0} \rightarrow \mathbb{R}^{s_{L+1}}$ denotes a nonlinear function (which is unknown and we want to estimate) and $\epsilon^i_j$ denotes a Gaussian noise independent of $\boldsymbol{x}^i_j$ i.e. $\epsilon^i_j \stackrel{\text{i.i.d}}{\sim}\mathcal{N}(0, \sigma^2_\epsilon)$ with variance $\sigma^2_{\epsilon}$. We assume that $f^{i}(\bullet)$ is $\beta$-Hölder-smooth functions \citep{takezawa2005introduction} and the intrinsic dimension of each of the client's data is $\boldsymbol\Lambda$.
We assume each client has the same fully-connected deep neural network (DNN), each having their individual neural network parameters. We denote the output of the neural network as $f_{\boldsymbol{\theta}}(\bullet)$ where $\boldsymbol{\theta} \in \mathbb{R}^T$ represents the network parameters. Analogously the output of the $i$-th client is denoted as $f^i_{\boldsymbol{\theta}}$.  
Each neural network has $L$ hidden layers, where the $j$-th hidden layer has $s_j$ neurons and its corresponding activation functions $\sigma(\bullet)$. We denote $\boldsymbol{\mathcal{S}} = \{s_i\}_{i=1}^{L}$, the set of all neurons in the neural network. The assumption made is that all the neural network parameters are bounded {\em i.e.}, $||\boldsymbol{\theta}||_{\infty} \leq \Omega$ for $\Omega >0$.

The main aim is to architect a Federated Learning system that takes into consideration the coreset problem for each individual client along with personalization and tackling overfitting. Using the above problem setting we first formulate the client side objective and server side objective for the Federated Learning setting as a bi-level optimization problem.

\subsection{Federated Learning Objectives}
 The standard BNN model \citet{jordan1999introduction} aims to solve the optimization problem to find the closest distribution $q^i(\boldsymbol{\theta})$ for the $i$-th client from the family of distributions $\mathcal{Q}_i$ to match  the posterior distribution $\pi(\boldsymbol{\theta}|\boldsymbol{\mathcal{D}^i})$ of the given data $\boldsymbol{\mathcal{D}^i}$ via minimizing the KL-divergence as follows.

\textbf{Client Side Objective}
\begin{equation}        \mathcal{F}_{i}(z) \triangleq 
\min_{q^i(\boldsymbol{\theta}) \in \mathcal{Q}_i}\mathbb{D}_{KL}(q^i(\boldsymbol{\theta})||\pi(\boldsymbol{\theta}|\boldsymbol{\mathcal{D}^i})) \Leftrightarrow \min_{q^i(\boldsymbol{\theta}) \in \mathcal{Q}_i} \overbrace{- \mathbb{E}_{q^i(\boldsymbol{\theta})}[\log P_{\boldsymbol{\theta}}(\boldsymbol{\mathcal{D}^i})]}^{\text{reconstruction error over } \boldsymbol{\mathcal{D}}} +  \zeta \overbrace{\mathbb{D}_{KL}(q^i(\boldsymbol{\theta}) || \pi(\boldsymbol{\theta}))}^{\text{regularization term}}
\label{eq:originalPFedBayes}
\end{equation}
Here $\pi(\boldsymbol{\theta})$ denotes the prior distribution and $P_{\boldsymbol{\theta}}(\boldsymbol{\mathcal{D}^i})$ denotes the likelihood and $\zeta$ is a personalization constant that defines the weightage towards more regularization thus leading to more personalization. 

\textbf{Server Side Objective}
On the server side the global model  tries to find the closest distribution in $\mathcal{Q}_z$ to the client's distribution by minimizing the aggregrate KL divergence from all the clients as follows:
\begin{equation}
    \min_{z(\boldsymbol{\theta}) \sim \mathcal{Q}_z} \mathcal{F}(z) \triangleq \frac{1}{N}\sum_{i=1}\mathcal{F}_{i}(z)
    \label{eq:ztheta}
\end{equation}

%\gr{Why cant eqns (1) and (2) be inlined/moved into appendix?  Perhaps a table of notations will help to realize that $w$ is only invoked in the Bayesian coreset setting in the following section. You could save space by perhaps inlining some of the equations and also merging some.}

\subsection{Bayesian Coresets Optimization}
We now introduce the notion of coreset weights i.e. we assign to each client $i$'s data $\boldsymbol{\mathcal{D}^i}$ a weight vector $\boldsymbol{w}_i \in \mathbb{R}^n$ that will act as the corresponding coreset weight for the $i$-th client.
In standard bayesian coresets optimization setting, the goal is to control the deviation of coreset log-likelihood from the true log-likelihood via some sparsity ($n_k << n$). In accordance to \citep{zhang2021bayesian} we utilise the following optimization objective for the $i$-th client:
\begin{equation}
\begin{aligned}
 \arg \min_{\boldsymbol{w}_i \in \mathbb{R}^n} \mathcal{G}^i(\boldsymbol{w}_i) := \left\lVert\mathcal{P}_{\boldsymbol{\theta}}(\boldsymbol{\mathcal{D}^i}) - \mathcal{P}_{\boldsymbol{\theta},\boldsymbol{w}_i}(\boldsymbol{\mathcal{D}^i})\right\rVert^2_{\hat{\pi},2} \hspace{7pt} s.t. \hspace{6pt} ||\boldsymbol{w}_i||_0 \leq n_k,  \hspace{8pt} \forall i \in [N]
    \end{aligned}
    \label{eq:Bayesian Coreset original Eq}
\end{equation}
where the coreset weights $\boldsymbol{w}_i$ are considered over the data points $\boldsymbol{\mathcal{D}^i}$ and $L^2(\hat{\pi})$-norm as the distance metric is considered in the embedding Hilbert Space.
Specifically, $\hat{\pi}$ is the weighting distribution that has the same support as true posterior $\pi$.
The above equation can be further approximated to the following where $\hat{g_j}$ is a Monte-Carlo approximation over $g_j = \mathcal{P}_{\theta}(\boldsymbol{\mathcal{D}}^i_j) - \mathbb{E}_{\theta \sim \hat{\pi}}P_{\theta}(\boldsymbol{\mathcal{D}}^{i}_j)$ for Monte-carlo samples (the derivation can be found in Appendix \ref{prop:AIHT_simplified})
\begin{equation}
\begin{aligned}
     \arg \min_{\boldsymbol{w} \in \mathbb{R}^n} \mathcal{G}^i(\boldsymbol{w}_i) :=  \left\lVert\sum_{i=1}^n \hat{g_j} - \sum_{i=1}^n \boldsymbol{w}_i\hat{g_j} \right\rVert_2^2  \hspace{6pt} s.t. \hspace{8pt} ||\boldsymbol{w}_i||_0 \leq n_k,  \hspace{8pt} \forall i \in [N]
\end{aligned}
\end{equation}

The above \textit{sparse regression} problem is non-convex due to the combinatorial nature of the constraints. \citep{campbell2018bayesian} uses the $l_2$-norm formulation which however results in less approximation accuracy compared to \citep{campbell2019sparse}. As the authors in \citep{zhang2021bayesian} point out both the above approaches have expensive computation cost and hence they propose a better alternative via accelerated iterative thresholding Appendix \ref{alg: AIHT}.

%In this case in total we have $N$ objective functions and for each individual client $i$ we want to minimize $\mathcal{G}^{i}(\boldsymbol{w_i})$. 

%\textbf{Methods to solve this}

%\citep{zhang2021bayesian} proposes methods to solve this from a non-convex optimization perspective.

%\begin{itemize}
 %   \item \textbf{Iterative Hard Thresholding} The classical IHT \citep{blumensath2009iterative} is a
%projected gradient descent method that performs a gradient descent step and then projects the iterate
%onto the non-convex k-sparsity constraint set.
 %   \item \textbf{Accelerated Iterative Hard Thresholding} is an incremental improvement over vanilla IHT that takes into account best step size selection using line search.
%\end{itemize}

%The original paper \citep{zhang2021bayesian} has compared the above two algorithms with state of the art methods like SparseVI \citep{campbell2019sparse} and GIGA \citep{campbell2018bayesian}. As part of our experiments incorporating bayesian coresets with FL, we will also be implementing these algorithms. 

%Once we get the weight vectors we then go ahead to solve the federated Learning objective as follows:

\section{Methodology - $\methodprop$}

%\subsection*{Bayesian Coresets for Clients under Federated Learning}

We now proceed towards formulating our problem for combining the coreset optimization problem in a federated learning setting.

\subsection{Modified Federated Learning Objectives - Incorporating Coreset Weights}

\textbf{Modified Client Side Objective}
 %Figure: \ref{fig:sysarch}
We now aim towards incorporating the coreset formulation from Eq: \ref{eq:Bayesian Coreset original Eq} in our federated learning setting Eq: \ref{eq:originalPFedBayes}. Assuming the bayesian coreset weights setup for each client $i$ from , we introduce a new modified client objective function
\begin{equation}
    \mathcal{F}_{i}^{w}(z) \triangleq \min_{q^i(\boldsymbol{\theta}) \sim \mathcal{Q}_i} [-\mathbb{E}_{q^i(\boldsymbol{\theta})} \log(P_{\boldsymbol{\theta}, \boldsymbol{w_i}}(\boldsymbol{\mathcal{D}^{i}})) + \zeta \mathbb{D}_{KL}(q^{i}(\boldsymbol{\theta}) || z(\boldsymbol{\theta}))]
    \label{eq:weighted_client-opt}
\end{equation}

where $z(\boldsymbol{\theta})$ and $q^{i}(\boldsymbol{\theta})$ denote the global distribution and the local distribution for the $i$-the client that is to be optimized respectively. Here, $\mathcal{F}_i^{w}(\bullet)$ and $\mathcal{F}_i(\bullet)$ indicates the coreset weighted i-th client objective and full data based client objective respectively.

In particular, let the family of distributions of the $i$-th client $\mathcal{Q}_i$and server $\mathcal{Q}_z$satisfy

\begin{equation}
    \mathcal{Q}_{i,k} \sim \mathcal{N}(\mu_{i,k}, \sigma^2_{i,k}), 
    \hspace{12pt}
        \mathcal{Q}_{z,k} \sim \mathcal{N}(\mu_{z,k}, \sigma^2_{z,k}) 
        \hspace{26pt} k = 1, \dots, T
    \label{eq:qtheta}
\end{equation}
where the above gaussian parameters correspond to the mean and variance for the $k$-th parameter of the $i$-th client respectively. Similar holds true for the server.
This is a valid assumption commonly used in literature
\citep{blundell2015weight} that let's us
simplify the evaluation of $KL$ divergence score $\mathbb{D}_{KL}$ between $q^i(\boldsymbol{\theta})$ and $z(\boldsymbol{\theta})$, resulting in a closed form solution.

\subsection{Novel Objective : Achieving near-optimal client distribution based on optimal coreset weights}

Let $\boldsymbol{z^*(\boldsymbol{\theta})}$ be the optimal variational solution of the problem in Eq: \eqref{eq:ztheta}
and $\hat{q^i}(\boldsymbol{\theta})$ be its corresponding variational solution for the $i$'th client's objective in Eq: \eqref{eq:originalPFedBayes}. Now, for the same $\boldsymbol{z^{*}(\boldsymbol{\theta})}$, let $\hat{q^{i}}(\boldsymbol{\theta}; \boldsymbol{w})$ denote the corresponding variational solution for the weighted coreset client objective in Eq: \ref{eq:weighted_client-opt}. We want to ensure that the optimal distribution $\hat{q^{i}}(\boldsymbol{\theta}; \boldsymbol{w})$ which minimizes the coreset objective does not deviate too much from the original distribution $\hat{q^i}(\boldsymbol{\theta})$ (which acts a solution for Eq: \ref{eq:originalPFedBayes}) and hence we want to fixate on $\boldsymbol{w}$ accordingly. This intuition comes from the fact that we want to choose $\boldsymbol{w}$ in such a way that the performance of our federated learning system in terms of accuracy does not drop too further in the face of training on a small percentage of data dictated by the coreset weights $\boldsymbol{w}$ (\textit{in our problem setting we denote coresets weights as $\boldsymbol{w}_i$ for each client $i$}).
From here forward, let $\mathcal{F}_{i}^{\boldsymbol{w}}(\bullet)_{\arg}$ denote  $\hat{q^{i}}(\boldsymbol{\theta}; \boldsymbol{w})$ and $\mathcal{F}_{i}(\bullet)_{\arg}$ 
denote $\hat{q^{i}}(\boldsymbol{\theta})$.

\subsection{Observations and Motivation for a New Objective}
\label{initial_hyp}
If we observe the client equation Eq: \ref{eq:originalPFedBayes} closely, we will see that for each client $i$ we are first fixing $z(\boldsymbol{\theta})$ in the equation. For a fixed $z(\boldsymbol{\theta})$, we now search for the optimal distribution $\hat{q^{i}}(\boldsymbol{\theta})$ among the whole family of distributions local to that client $\mathcal{Q}^{i}$. This is similar also for Eq: \ref{eq:weighted_client-opt} for the coreset weighted client objective.

Our objective is the learned optimal local distribution for the general client optimization objective should be as close as possible to that of the weighted coreset based client optimization objective for the same value of $z$.

So we want to formulate a new objective function such that for each client $i$ we minimize the divergence between the two optimal distributions resulting from the coreset and normal objective functions.
\begin{equation}
        \{\boldsymbol{w_i}^*\} \triangleq \arg \min_{\boldsymbol{w}} \mathbb{D}_{KL}( \mathcal{F}_i^{\boldsymbol{w}}(z)_{\arg} || \mathcal{F}_i(z)_{\arg})
        \hspace{6pt}
        \Leftrightarrow\hspace{6pt} \arg \min_{\boldsymbol{w}}\mathbb{D}_{KL}(  \hat{q^{i}}(\boldsymbol{\theta},\boldsymbol{w}) || \hat{q^{i}}(\boldsymbol{\theta}))        \hspace{12pt}            \left\|\boldsymbol{w_i}\right\|_{0} \leq n_k
\label{eq: individual Kl}
\end{equation}
\subsection{ Incorporating likelihood into Modified Client Objective}

Although the above minimization objective Eq: \ref{eq: individual Kl} captures the intuition behind matching the near-optimal performance (accuracy) with only a small coreset of the original client's data to that by using the entire data, this approach does not take into account how close the likelihood of each of the client's coreset weighted data subset $P_{\boldsymbol{\theta}, \boldsymbol{w_i}}(\boldsymbol{\mathcal{D}^i})$ is to that of the original client's data $P_{\boldsymbol{\theta}}(\boldsymbol{\mathcal{D}^{i}})$.

More specifically, we now want to choose the optimal coreset weights  ($\boldsymbol{w_i}$ (personal coreset weights) by not only minimizing the KL divergence between the corresponding client distributions ( $\hat{q^{i}}(\boldsymbol{\theta}; \boldsymbol{w})$ and  $\hat{q^{i}}(\boldsymbol{\theta})$) but also taking into account the closeness of the coreset weighted data likelihood to that of the original likelihood.
\begin{equation}
\begin{aligned}
 \{\boldsymbol{w_i}^*\} \triangleq \arg \min_{\boldsymbol{w}}\mathbb{D}_{KL}(\hat{q^{i}}(\boldsymbol{\theta}, \boldsymbol{w}) || \hat{q^{i}}(\boldsymbol{\theta}) ) + \left\lVert P_{\boldsymbol{\theta}}(\boldsymbol{\mathcal{D}^{i}}) - P_{\boldsymbol{\theta}, \boldsymbol{w_i}}(\boldsymbol{\mathcal{D}^i})\right\rVert^2_{\hat{\pi},2} 
\hspace{7pt} ||\boldsymbol{w_i}||_{0} \leq n_k 
        \label{eq:modified-personalized-weights}
\end{aligned}
\end{equation}

\section{Algorithm}

 % \col{name1 & street1 & address1 & post code 1}

\begin{table*}

\begin{tabularx}{\textwidth}
{l}
\hline \textbf{Algorithm 1}: $\methodprop$ \\
\toprule
\resizebox{\textwidth}{!}{%
\begin{tabular}{@{}|l|l|@{}}
\toprule
\textbf{\textcolor{purple}{Server side Objective \& Coreset Update}} & \textcolor{purple}{\textbf{Client Side Objective}} \\
\hline \textbf{Cloud server executes}: & \textcolor{brown}{\textbf{ClientUpdate}} $\left(i, w_i,\boldsymbol{v}^t\right):$ \\
$\quad$ \textbf{Input} $T, R, S, \lambda, \eta, \beta, b, \boldsymbol{v}^0=\left(\boldsymbol{\mu}^0, \boldsymbol{\sigma}^0\right)$ & $\quad$ $\boldsymbol{v}_{z, 0}^t=v^t$ \\
$\quad$ \textbf{for} $t=0,1, \ldots, T-1$ \textbf{do} &$\quad$ \textbf{for} $r=0,1, \ldots, R-1$ \textbf{do}  \\
$\quad$ $\quad$ \textbf{for} $i=1,2, \ldots, N$ \textbf{in parallel do} & $\quad$ $\quad$ $\boldsymbol{D}_{\Lambda}^i \leftarrow$ sample a minibatch 
$\Lambda$ with size $b$ from $\boldsymbol{D}^i$ \\
$\quad$ $\quad$ $\quad \boldsymbol{v}_i^{t+1} \leftarrow$ \textcolor{teal}{\textbf{CoresetOptUpdate}} $\left(i, \boldsymbol{v}^t\right)$ & $\quad$ $\quad$ $\boldsymbol{D}_{\Lambda, w}^i \leftarrow$ sample a minibatch $\Lambda$ with size $b$ from $w_i\boldsymbol{D}^i$ \\
$\quad$ $\quad$ $\mathbb{S}^t \leftarrow$ Random subset of clients with size $S$ & $\quad$ $\quad$ $\boldsymbol{g}_{i, r} \leftarrow$ Randomly draw $K$ samples from $\mathcal{N}(0,1)$ \\
$\quad$  $\quad$ $\boldsymbol{v}^{t+1}=(1-\beta) \boldsymbol{v}^t+\frac{\beta}{S} \sum_{i \in S^t} \boldsymbol{v}_i^{t+1}$ & $\quad$ $\quad$  $\Omega^i\left(\boldsymbol{v}_r^t\right) \leftarrow$ Use Eq:\ref{eq:weighted-stochastic-estimator} with 
$\boldsymbol{g}_{i, r}, \boldsymbol{D}_{\Lambda}^i$ and $\boldsymbol{v}_r^t$\\
\textcolor{teal}{\textbf{CoresetOptUpdate}}$\left(i,\boldsymbol{v}^t\right):$ & $\quad$ $\quad$  $\nabla_v \Omega^i\left(\boldsymbol{v}_r^t\right) \leftarrow$ Back propagation w.r.t $\boldsymbol{v}_r^t$\\
$\quad$ $y = P_{\theta}(D^{i}),\Phi w^i =P_{\theta,w^i}(D^i) $ & $\quad$ $\quad$  $\boldsymbol{v}_r^t \leftarrow$ Update with $\nabla_v \Omega^i\left(\boldsymbol{v}_r^t\right)$ using GD algorithms\\
$\quad$\textbf{ Objective} $f(\boldsymbol{w}) = \mathbb{D}_{KL}(q^i_w|| q^i) + ||y-\Phi w^i||^2_2$ & $\quad$ $\quad$  $\Omega_z^i\left(\boldsymbol{v}_{z, r}^t\right) \leftarrow$ Forward propagation w.r.t $\boldsymbol{v}$\\
$\quad$ $t=0, l_0 = 0, w^i_0 = 0$ & $\quad$ $\quad$ $\nabla \Omega_z^i\left(\boldsymbol{v}_{z, r}^t\right) \leftarrow$ Back propagation w.r.t $\boldsymbol{v}$\\
$\quad  w_i \leftarrow  w^i_0 , v^t_{z,R} \leftarrow 0$ & $\quad$ $\quad$ Update $\boldsymbol{v}_{z, r+1}^t$ with $\nabla \Omega_z^i\left(\boldsymbol{v}_{z, r}^t\right)$ using GD algorithms\\
$\quad$ \textbf{repeat} & $\quad$ $\quad$ \textit{Repeat the above 7 steps for} \\
$\quad \quad v^t_z, q^i, q^i_w  \leftarrow \textbf{\textcolor{brown}{\text{ClientUpdate}}}\left(i, w_i,\boldsymbol{v}^t\right)$ & $\quad$ $\quad$ \textit{the weighted stochastic estimate}\\
$\quad \quad f(w) = \mathbb{D}_{KL}(q^i_w|| q^i) + ||y-\Phi w^i||^2_2$ & $\quad \hat{q^i(z)} \leftarrow \arg \Omega_z^i\left(\boldsymbol{v}_{z, R}^t\right)$ 
$, \hat{q^i_w(z)} \leftarrow \arg \Omega_z^i\left(\boldsymbol{v}_{w,z, R}^t\right)$\\
$\quad \quad  \textbf{\textcolor{blue}{\text{Accelerated-IHT}}} (f(w))  (\text{Algo \ref{algorithm:AIHT}})$ & $\quad$ \text{return } $\boldsymbol{v}_{z, R}^t,\hat{q^i(z)}, \hat{q^i_w(z)}$\\
$\quad$ until Stop criteria met &
 return $\boldsymbol{v}_{z, R}^t $ to the cloud server \\
\hline
\end{tabular}%
}
\end{tabularx}
\end{table*}

We showcase our complete algorithm for $\methodprop$ in Table 1 below
In line with \citep{zhang2022personalized} we utilise a reparameterization trick for $\boldsymbol{\theta}$ via variables $\boldsymbol{\mu}$ and $\boldsymbol{\rho}$ i.e. $\boldsymbol{\theta}=h(\boldsymbol{v}, \boldsymbol{g})$, where $\boldsymbol{\theta}_m =h\left(v_m, g_m\right) =\mu_m+\log \left(1+\exp \left(\rho_m\right)\right) \cdot g_m, \hspace{7pt} g_m \sim \mathcal{N}(0,1)$, where $m \in [1, \dots , T]$.
For the first term in Equation  \ref{eq:weighted_client-opt}, we use a  minibatch stochastic gradient descent to get an estimate for the $i$-th client as follows:

%\begin{equation}
%\begin{alined}
%\Omega^i(\boldsymbol{v}) \approx-\frac{n}{b} \frac{1}{K} \sum_{j=1}^b \sum_{k=1}^K \log p_{h\left(\boldsymbol{v}, \boldsymbol{g}_k\right)}^i\left(\boldsymbol{D}_j^i\right) \\
%+\zeta \operatorname{KL}\left(q_{\boldsymbol{v}}^i(\boldsymbol{\theta}) \| z_{\boldsymbol{v}}(\boldsymbol{\theta})\right)
%\end{aligned}
%\label{eq:updateAlgoeq}
%\end{equation}

\begin{equation}
\begin{aligned}
    \Omega^i(\boldsymbol{v_w}) \approx-\frac{n}{b} \frac{1}{K} \sum_{j=1}^b \sum_{k=1}^K \log p_{h\left(\boldsymbol{v_w}, \boldsymbol{g}_k\right)}^i\left(\boldsymbol{D}_j^i\right)
+\zeta \mathbb{D}_{KL}\left(q_{\boldsymbol{v_w}}^i(\boldsymbol{\theta;w}) \| z_{\boldsymbol{v_w}}(\boldsymbol{\theta})\right)  
\end{aligned}
\label{eq:weighted-stochastic-estimator}
\end{equation}

where $b$ and $K$ are minibatch size and Monte Carlo sample
size, respectively.

%For computing the gradient with respect to $\boldsymbol{w}$, we consider the following set up.

%\pc{Probably move this lemma to appendix}
%\begin{lemma}
 %   \label{lemma:gradKL}
%The gradient with respect to $\boldsymbol{w}$ for Eq: \ref{eq:AIHT-modified} i.e.

    %$$\nabla_{w}\mathbb{D}_{KL}( \hat{q^{i}(\boldsymbol{\theta};\boldsymbol{w})} ||  \hat{q^{i}(\boldsymbol{\theta})})    \Leftrightarrow 
 %\int_{\boldsymbol{\theta}}\nabla_w \hat{q^{i}(\boldsymbol{\theta};\boldsymbol{w})}\biggl[ \log \hat{q^{i}(\boldsymbol{\theta};\boldsymbol{w})} + 1 - \log \hat{q^i(\boldsymbol{\theta})}\biggr]d\boldsymbol{\theta}$$

%where the individual components of the gradient term $\nabla_w \hat{q^{i}(\boldsymbol{\theta};\boldsymbol{w})}$ is given by

%\begin{equation*}
    %\frac{\partial\hat{q^i(\boldsymbol{\theta}_j; \boldsymbol{w})}}{\partial w} = \frac{\int_{\boldsymbol{\theta}}\biggl(\hat{q^i(\boldsymbol{\theta}; \boldsymbol{w})}\frac{1}{\mathcal{P}_{\theta,w}(\boldsymbol{\mathcal{D}}^i)}\frac{\partial\mathcal{P}_{\theta,w}(\boldsymbol{\mathcal{D}}^i)}{\partial w})\biggr)d\boldsymbol{\theta}}{    \zeta\int_{\boldsymbol{\theta}}\biggl(\log  \hat{q^{i}(\boldsymbol{\theta})} + 1 -\log(z^*(\boldsymbol{\theta}))\biggr)d\boldsymbol{\theta}} 
%\end{equation*}
%\end{lemma}

%The proof for the above Lemma and its subsequent simplification can be found in Appendices section.
Here $R$ indicates the number iterations after which the clients upload the localized global models to the server. Like \citep{t2020personalized}, we use an  additional parameter $\beta$ in order to make the algorithm converge faster.

\section{Theoretical Contributions} 
\label{sec: Theoretical Analysis}

Here we provide theoretical analysis related to the averaged generalization error for $\methodprop$ w.r.t our baseline $\baseline$. %and how the generalization error is close to that of $\baseline$ based on the coreset size selection $k$ and as a function of $\boldsymbol{w}$ . 
The main results and derivations of the proofs can be found in the Appendix \ref{supplementary:proofs}. We first provide certain definitions here.

\begin{definition}\label{def:HellingerDistance}
The Hellinger distance for a particular client $i$ between the estimate likelihood $\mathcal{P}_{\boldsymbol{\theta}}^i$ and the true likelihood $\mathcal{P}^i$ is defined as follows $d^2(\mathcal{P}_{\theta}^i, \mathcal{P}^i) = \mathbb{E}_{X^i}(1-e^{-\frac{\bigl[f^i_{\boldsymbol\theta}(X^i) - f^i(X^i)\bigr]^2}{8\sigma^2_{\epsilon}}})$
\end{definition}

%As per Theorem 1. of \citep{zhang2022personalized}, the term $ \frac{1}{N}\sum_{i=1}^N\int_{\Theta}d^2(\mathcal{P}^i_{\boldsymbol{\theta}, w}, \mathcal{P}^i)\hat{q^i}(\boldsymbol{\theta})d\boldsymbol{\theta}$ (where $\hat{q^i}(\boldsymbol{\theta})$ is optimal solution to Eq: \ref{eq:originalPFedBayes} ) is upper bounded by the estimation error of the model as well as an approximation error (See Appendix).
 
Let $\hat{q^i}(\boldsymbol{\theta};\boldsymbol{w})$ be the corresponding variational solution for the
i-th client’s subproblem under the coreset weighted regime and let us define the following term \textbf{Generalization Error Term}:
$\int_{\Theta}d^2(\mathcal{P}^i_{\boldsymbol{\theta}, w}, \mathcal{P}^i)\hat{q^i}(\boldsymbol{\theta};\boldsymbol{w})d\boldsymbol{\theta}$ as the $i$-th client's generalization error.

\begin{theorem}\label{Theorem:1}
The difference in the upper bound incurred in the overall generalization error of $\methodprop$ as compared w.r.t that of $\baseline$ is always upper bounded by a closed form positive function that depends on the coreset weights and coreset size- $\boldsymbol{\Im}(\boldsymbol{w}, n_k)$. generalization error in the original full data setup 
 \begin{align*}
    \left[\frac{1}{N}\sum_{i=1}^N\int_{\Theta}d^2(\mathcal{P}^i_{\boldsymbol{\theta}}, \mathcal{P}^i)\hat{q^i}(\boldsymbol{\theta})d\boldsymbol{\theta}\right]_{u.b.}
   -  \left[\frac{1}{N}\sum_{i=1}^N\int_{\Theta}d^2(\mathcal{P}^i_{\boldsymbol{\theta}, w}, \mathcal{P}^i)\hat{q^i}(\boldsymbol{\theta};\boldsymbol{w})d\boldsymbol{\theta}\right]_{u.b.} \leq \boldsymbol{\Im}(\boldsymbol{w}, n_k)
    \end{align*}
\end{theorem}

 This indicates that the extra estimation error and approximation model incurred by the coreset weighted objective is a direct function of the coreset weight and thus the coreset size and hence can be measured in closed form. Proof. in Appendix \ref{thm1:Supplementary}

\begin{theorem}\label{Theorem:2}
The convergence rate of the generalization error under $L^2$ norm of $\methodprop$ is minimax optimal up to a logarithmic term (in order $n_k$) for bounded functions ($\boldsymbol\beta$-Hölder-smooth functions) $\{f^i\}^N_{i=1}$,  $\{f^i_{\boldsymbol{\theta}}\}^N_{i=1}$ and $\{f^i_{\boldsymbol{\theta},\boldsymbol{w}}\}^N_{i=1}$ where $C_2$, $C_3$ and $\delta^{\prime}$ are constants (defined in Appendix ) and $\boldsymbol\Lambda$ being the intrinsic dimension of each client's data:

\begin{align*}
\frac{C_F}{N} \sum_{i=1}^N \int_{\boldsymbol{\theta}}\left\|f_{\boldsymbol{\theta},\boldsymbol{w}}^i\left(X^i\right)-f^i\left(X^i\right)\right\|_{L^2}^2 \hat{q^i}(\boldsymbol{\theta}; \boldsymbol{w}) d \boldsymbol{\theta} 
\leq C_2 n_k^{-\frac{2 \beta}{2 \beta+\boldsymbol{\Lambda}}} \log ^{2 \delta^{\prime}}(n_k) .
\end{align*}
and
\begin{align*}
\inf _{\left\{\left\|f_{\boldsymbol{\theta}, \boldsymbol{w}}^i\right\|_{\infty} \leq F\right\}_{i=1}^N\left\{\left\|f^i\right\|_{\infty} \leq F\right\}_{i=1}^N} \frac{C_F}{N} \sum_{i=1}^N 
\int_{\boldsymbol{\theta}}\left\|f_{\boldsymbol{\theta}, \boldsymbol{w}}^i\left(X^i\right)-f^i\left(X^i\right)\right\|_{L^2}^2 \hat{q}^i(\boldsymbol{\theta};\boldsymbol{w}) d \boldsymbol{\theta} \geq C_3 n_k^{-\frac{2 \beta}{2 \beta+\boldsymbol\Lambda}}
\end{align*}

where $n_k$ denotes the coreset size per client dataset and $n$ denotes the original per client dataset size and $\frac{d^2\left(P_{\boldsymbol{\theta}, \boldsymbol{w}}^i, P^i\right)}{\left\|f_{\boldsymbol{\theta},\boldsymbol{w}}^i\left
(X^i\right)-f^i\left(X^i\right)\right\|_{L^2}^2} \geq \frac{1-\exp \left(-\frac{4 F^2}{8 \sigma_\epsilon^2}\right)}{4 F^2} \triangleq C_F$.
\end{theorem}

This indicates that the minimax optimality of the generalization error for $\methodprop$ is in logarithmic bounds w.r.t the coreset size $n_k$. Proof. in Appendix \ref{ref:theorem2:Supplementary}

\begin{theorem}\label{Theorem:3}
The lower bound (l.b.) incurred for the deviation for the weighted coreset $\methodprop$ (\ref{eq:weighted_client-opt}) generalization error is always higher than the lower bound of that for the original $\baseline$ objective (\ref{eq:originalPFedBayes}) with a delta difference (\textbf{Error I} - \textbf{Error II}) as $\mathcal{O}(n_k^{-\frac{2 \beta}{2 \beta+\boldsymbol\Lambda}})$ 
\begin{align*}
\Bigg[\underbrace{\sum_{i=1}^N\int_{\boldsymbol{\Theta}}\left\|f_{\boldsymbol{\theta}, \boldsymbol{w}}^i\left(X^i\right)-f^i\left(X^i\right)\right\|_{L^2}^2 \hat{q^i}(\boldsymbol{\theta}, \boldsymbol{w}) d \boldsymbol{\theta}}_{\text{Coreset weighted objective Generalization Error (\textbf{Error I})}} \Bigg]_{l.b.} 
> \Bigg[\underbrace{\sum_{i=1}^N\int_{\boldsymbol{\Theta}}\left\|f_{\boldsymbol{\theta}}^i\left(X^i\right)-f^i\left(X^i\right)\right\|_{L^2}^2 \hat{q^i}(\boldsymbol{\theta}) d \boldsymbol{\theta}}_{\text{Vanilla objective Generalization Error (\textbf{Error II})}}\Bigg]_{l.b.}
\end{align*}
\end{theorem}
This simply implies that the generalization error suffers in the case due to limited coreset samples but that is bounded in closed form w.r.t. the coreset sample size. Proof. in Appendix: \ref{ref: Theorem deviation lower bound}

\begin{theorem}\label{Theorem:4}
The lower bound incurred in the overall generalization error across all $N$ clients of $\methodprop$ is always higher compared to that of the generalization error in the original full data setup 
 \begin{align*}
    \left[\frac{1}{N}\sum_{i=1}^N\int_{\Theta}d^2(\mathcal{P}^i_{\boldsymbol{\theta}, w}, \mathcal{P}^i)\hat{q^i}(\boldsymbol{\theta};\boldsymbol{w})d\boldsymbol{\theta}\right]_{l.b.} \geq 
    \left[\frac{1}{N}\sum_{i=1}^N\int_{\Theta}d^2(\mathcal{P}^i_{\boldsymbol{\theta}}, \mathcal{P}^i)\hat{q^i}(\boldsymbol{\theta})d\boldsymbol{\theta}\right]_{l.b.}
    \end{align*}
\end{theorem}

 Both Theorem \ref{Theorem:1} and \ref{Theorem:4} implies that the overall spread of the Generalization Error Term in case of coreset weighted objective $\methodprop$ is much more wider than that of the original $\baseline$ case. Proof. in Appendix \ref{ref: Supplementary Theorem:4}

\section{Experiments}

Here we perform our experiments to 
showcase the utility of our method 
$\methodprop$ compared to other baselines like $\baseline$ and do an in-depth analysis of each of the components involved as follows.

\subsection{Utility of A-IHT  in Vanilla Bayesian Coresets Optimization}

In Algorithm 1, since we are employing A-IHT algorithm during coreset updates, we first want to study the utility of applying A-IHT (Accelerated Iterative Thresholding) in a simplistic Bayesian coresets setting using the algorithm proposed by \citep{huang2022coresets}.
For analysis we test the same on Housing Prices 2018 \footnote{https://www.gov.uk/government/statistical-data-sets/price-paid-data-downloads}{data}.
Figure \ref{fig:coresets} showcases the experiments done on the dataset using a riemann linear regression for different coreset sizes of the data (k = 220, 260, 300). As it can be seen the radius capturing the weights w.r.t to the coreset points matches closely at k=300 with that of the true posterior distribution (extreme right), thereby indicating the correctness of approximation and recovery of the true posterior by the A-IHT algorithm.

\begin{figure}[h!]
\centering
\begin{minipage}{.25\textwidth}
  \centering
  \includegraphics[width=.9\linewidth]{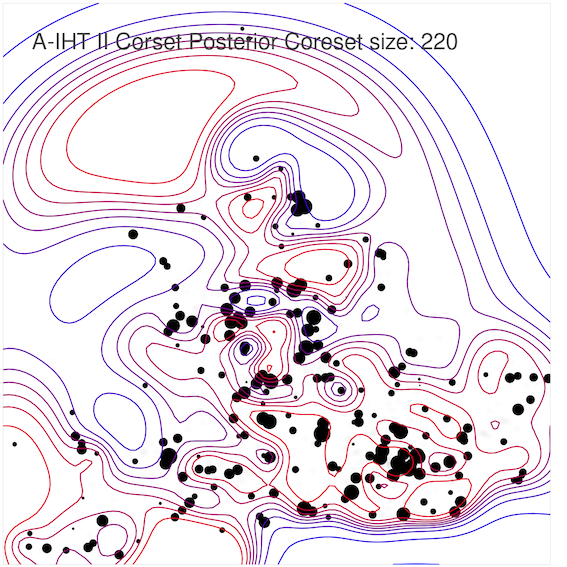}
  \label{fig:coreset1}
\end{minipage}%
\begin{minipage}{.25\textwidth}
  \centering
  \includegraphics[width=.9\linewidth]{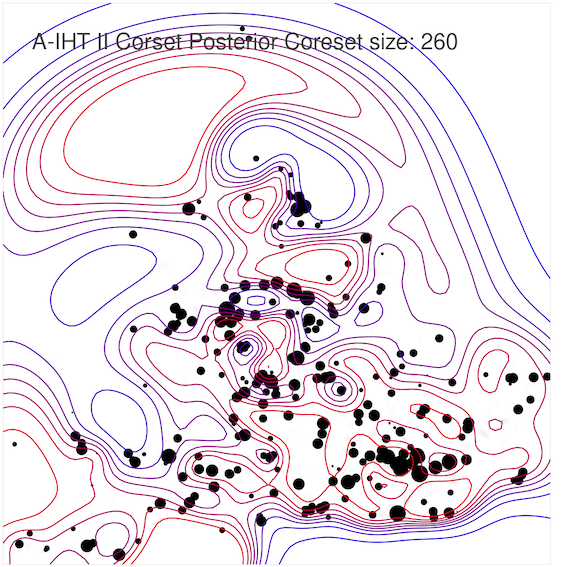}
  \label{fig:coreset2}
\end{minipage}%
\begin{minipage}{.25\textwidth}
  \centering
  \includegraphics[width=.9\linewidth]{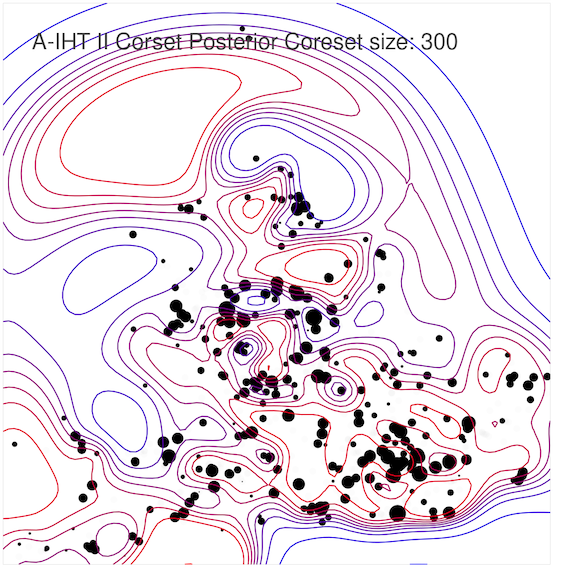}
  \label{fig:coreset3}
\end{minipage}%
\begin{minipage}{.25\textwidth}
  \centering
  \includegraphics[width=.9\linewidth]{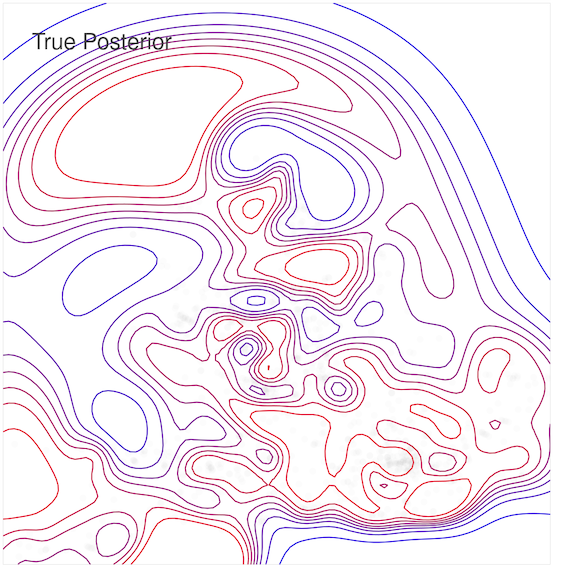}
  \label{fig:coreset4}
\end{minipage}
 \caption{Experiments on Bayesian reimann linear function regression for different settings of coreset size =220,260,300 constructed by Accelerated IHT II. 
        Coreset points are presented as black dots, with their radius indicating assigned weights. Extreme right showcases the true posterior distribution}
\label{fig:coresets}
\end{figure}

\subsection{Experiments on $\methodprop$ against S.O.T.A Federated Learning methods}

%Figure 1: The experiments show that the parameters learnt after taking a coreset of the data tend to match up with parameters, when trained on the full dataset. Hence, the KL Divergence loss is decreasing which is what we expected

%Figure 2: We are comparing a random subset selection of the data vs when a subset is selected using our method, the early convergence of the loss shows the success of our method. When trained on a random subset of data, the model takes more epochs in converging to the loss. 

Here, we compare the performance of the proposed method $\methodprop$ with a variety of baselines such as FedAvg \citep{pmlr-v54-mcmahan17a}, BNFed \citep{pmlr-v97-yurochkin19a}, pFedMe \citep{NEURIPS2020_f4f1f13c}, perFedAvg \citep{NEURIPS2020_24389bfe}, PFedBayes \citep{zhang2022personalized} based on non-i.i.d. datasets. We generate the non-i.i.d. datasets based on three public benchmark datasets, MNIST \citep{726791}, FMNIST (Fashion- MNIST) \citep{xiao2017fashionmnist} and CIFAR-10 \citep{krizhevsky2009learning}. For MNIST, FMNIST and CIFAR-10 datasets, we follow the non-i.i.d. setting strategy in \citep{t2020personalized}. In our use case we considered total 10 clients each of whom holds a unique local data.

\vspace{6pt}

\begin{table*}[h!]
\centering
\caption{Comparative results of personal and global accuracies (in \%) across all 7 methods}
\label{Table : Accuracy Stats}
\resizebox{\textwidth}{!}{%
\begin{tabular}{@{}c|cc|cc|cc@{}}
\toprule
{\textbf{Method}  (Percentage = sampling fraction) } & \multicolumn{2}{c}{\textbf{MNIST}}                                                                                                                    & \multicolumn{2}{c}{\textbf{FashionMNIST}}                                                                                                                                           & \multicolumn{2}{c}{\textbf{CIFAR}}                                                   \\ \cmidrule(l){2-7} 
                                   & \begin{tabular}[c]{@{}c@{}}Personal Model\end{tabular} & \begin{tabular}[c]{@{}c@{}}Global Model\end{tabular} & \begin{tabular}[c]{@{}c@{}}Personal Model\end{tabular} & \begin{tabular}[c]{@{}c@{}}Global Model\end{tabular} & \begin{tabular}[c]{@{}c@{}}Personal Model\end{tabular} & \begin{tabular}[c]{@{}c@{}}Global Model\end{tabular} \\ \cmidrule(r){1-7}
\textbf{FedAvg (Full/ 50\%)}            &    -                                    &  \textbf{\textcolor{magenta}{92.39}}(90.60)                                        & -                                              & 85.42(83.90)                                     & -                                                    &   \textbf{\textcolor{red}{79.05}}(56.73) \\
\textbf{BNFed (Full / 50\%)}  & - &  82.95(80.02) & - & 70.1(69.68) &- &  44.37(39.52)
\\
\textbf{pFedMe (Full / 50\%)}  & - & 91.25(89.67) & \textbf{\textcolor{orange}{92.02}}(84.71) & 84.41(83.45)  & \textbf{\textcolor{magenta}{77.13}}(66.75) & \textbf{\textcolor{magenta}{70.86}}(51.18)
\\
\textbf{perFedAvg (Full / 50\%)}  & \textbf{\textcolor{orange}{98.27}} & - & 88.51(84.90) & - &69.61(52.98) &-
\\
\textbf{$\baseline$ (Full / 50\%)}  & \textbf{\textcolor{red}{98.79}}(90.88) & \textbf{\textcolor{red}{97.21}}(92.33) & \textbf{\textcolor{red}{93.01}}(85.95) & \textbf{\textcolor{red}{93.30}}(82.33) &\textbf{\textcolor{red}{83.46}}(\textbf{\textcolor{orange}{73.94}}) &64.40(60.84)
\\
\textbf{$\randomsub$ (50\%)}  & 80.2 & 88.4 & 87.12 & \textcolor{magenta}{\textbf{90.75}} &48.31 &61.35
\\
\textbf{$\methodprop$ (k = 50\%)}  & \textbf{\textcolor{magenta}{92.48}} & \textbf{\textcolor{orange}{96.3}} & \textbf{\textcolor{magenta}{89.55}} & \textcolor{orange}{\textbf{92.7}} & 69.66 & \textbf{\textcolor{orange}{71.5}}
\\
\bottomrule
\end{tabular}%
\label{table:accuracy-report}
}
\subcaption[]{We report accuracies on both global and personal model for the current set of proposed methods across major datasets like \textbf{MNIST, CIFAR, FashionMNIST}. \textbf{\textcolor{red}{Red}} indicates the highest accuracy column-wise. Similarly \textbf{\textcolor{orange}{Orange}} and \textbf{\textcolor{magenta}{Magenta}} indicates the 2nd and 3rd best modelwise accuracy. (-) indicates no accuracy reported due to very slow convergence of the corresponding algorithm.
\textbf{Full indicates training on full dataset and 50\% is on using half the data size after randomly sampling 50\% of the training set.
}}
\end{table*}
\vspace{-15pt}

In Table \ref{Table : Accuracy Stats},  we showcase the accuracy statistics of the corresponding baselines discussed above with our method $\methodprop$ across two different configurations \textbf{Full Dataset Training} and \textbf{50\% Training} indicating half of the training samples were selected at random and then the corresponding algorithm was trained). As observed in almost all cases our method $\methodprop$ beats $\baseline$(random 50\% data sampled) by the following margins : +4.87\% on MNIST, +8.61\% on FashionMNIST, +9.71\% and almost on other baselines (random 50\% data sampled) and some baselines even when they were trained on full dataset (e.g. our method does better than PFedMe on FashionMNIST). In Appendix \ref{sec:Communication Complexity Analysis For Different Coreset Sizes} we showcase the communication complexity of our proposed method.

%Figure \ref{fig:coreset1}

%Figure \ref{fig:coreset2}

\begin{figure}[h!]
\centering
\begin{subfigure}{.45\textwidth}
  \centering
\includegraphics[width=.9\linewidth]{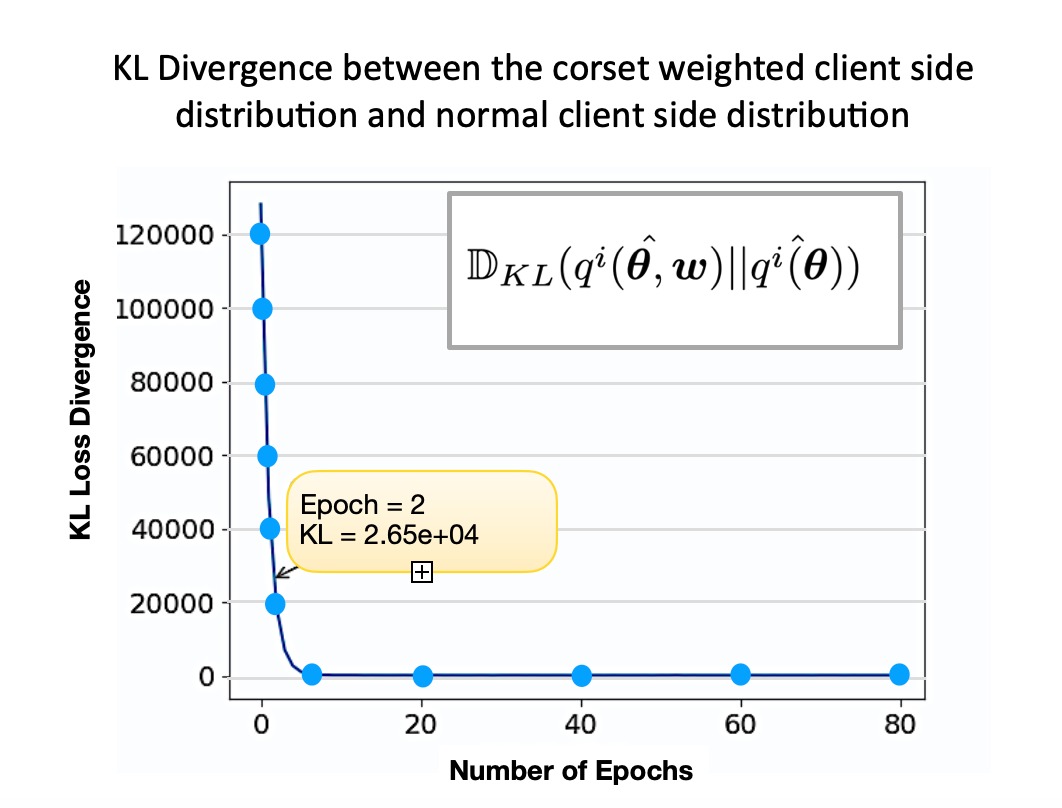}
\caption{The optimal distribution parameters learnt after taking a coreset of the data tend to match up with parameters, when trained on the full dataset after few epochs resulting in decreasing KL divergence score. This is in line with our initial hypothesis \ref{initial_hyp}}
  \label{fig:coreset1}
\end{subfigure}% 
\hfill
\begin{subfigure}{.45\textwidth}
  \centering
\includegraphics[width=.9\linewidth]{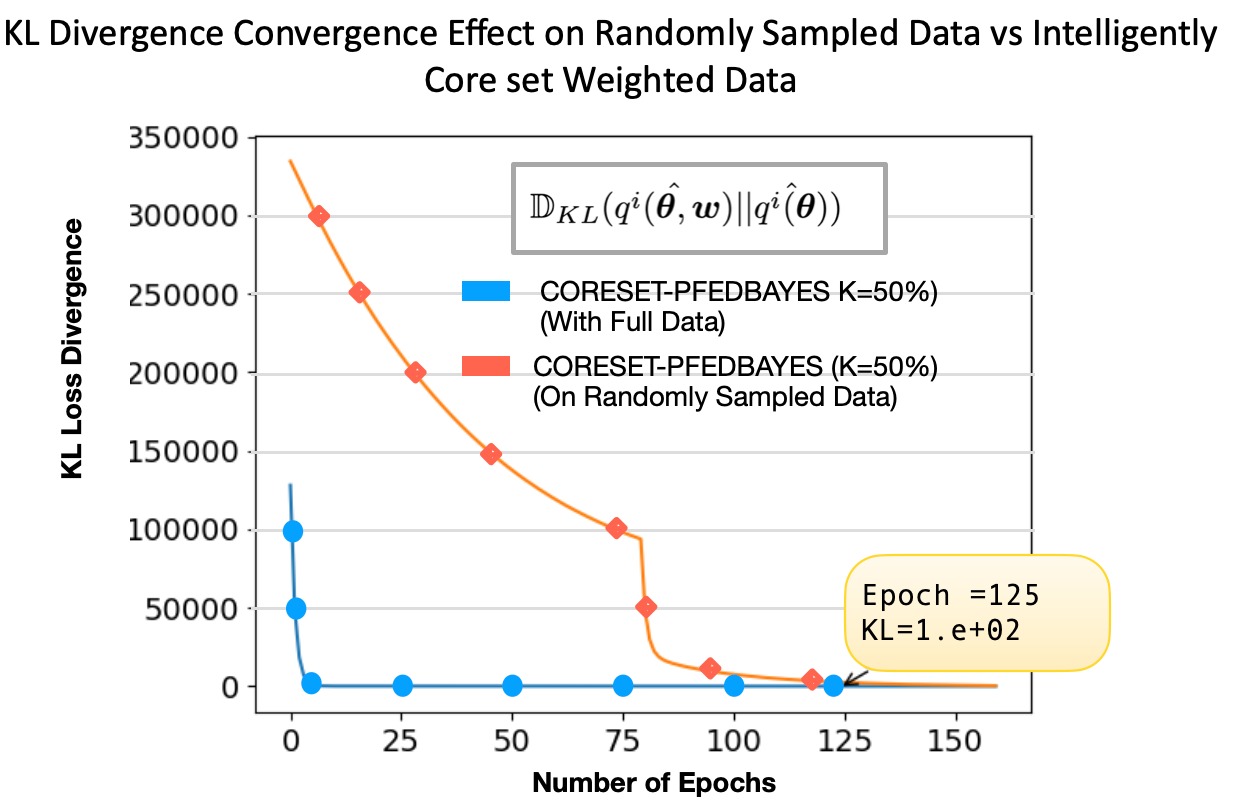}
\caption{We are comparing a random subset selection of the data vs when a subset is selected using our method, the early convergence of the KL-divergence shows the better performance of our method. When trained on a random subset of data, the model takes more epochs to converge.}
  \label{fig:coreset2}
\end{subfigure}%
 \caption{KL Divergence Plot over Number of Epochs (MNIST Dataset)}
\label{fig:coresets}
\end{figure}

% \begin{enumerate}[label={(\alph*)}]
%         \item \textbf{FedAvg} 
%         \item \textbf{PFedBayes} \citep{zhang2022personalized}proposes the original objective function \ref{eq:originalPFedBayes} during which for each client training is done on the entire client's individual data.        \item \textbf{Random Subset Selection} Here we only consider the fact that for each client we only consider a randomly subset subset of data for training purposes. We consider the sampling fraction $K$ as a tuning parameter for the percentage of data selected uniformly across all the clients' individual data.
%         \item \textbf{Diversity based Submodular Maximization Strategies for Subset Selection}
%         Here we compare with several diversity based subset selection based strategies that utilizes submodular maximization approach to compare with our proposed method $\methodprop$. To implement the same we utilised the \textit{Submodlib library} \footnote{ \href{github.com/decile-team/submodlib}{ 
%   - Submodlib decile library}} \citep{kaushal2022submodlib}.
% \end{enumerate}

\subsection{Experiments on $\methodprop$ against Submodular Subset Selection}

In order to showcase the advantage of model-centric subset selection methods over traditional data-centric(model-agnostic) methods like our proposed work $\methodprop$ which takes into account matching the client distribution under coreset setting $q^i(\boldsymbol{\theta}; \boldsymbol{w})$ to that in normal setting $q^i(\boldsymbol{\theta})$.
Hence, we compare our proposed method $\methodprop$ against \textbf{submodular based functions} (\textit{See Appendix \ref{Submodular_function} for definition})
(specifically, \textit{diversity} based submodular functions as the aim is to select a subset of data points that are most diverse). The discussion on some of the common diversity functions and their properties with regard to monotonicity and submodularity are provided in Appendix \ref{methods:submod}.
Each of the medical datasets consists of 3 classes out of which 1 class (Normal is kept as common) between 2 clients and the data about the other two classes is distributed separately to the 2 clients. Our aim here is to deploy our proposed method in this setting and compare against submodular based subset selection approaches.

More details  regarding the dataset description along with the experimental configurations can be found under Appendix. Note for each of the submodular function baselines, the data was sampled using a submodular function optimization strategy post which FedAvg was applied as the Federated Learning algorithm.

\begin{table*}[h!]
\centering
\caption{Comparative results of classwise global accuracies of all 9 methods on \textbf{3 different medical datasets} and \textbf{2 clients}}
\label{tab:acc}
\resizebox{\textwidth}{!}{%
\begin{tabular}{@{}c|ccc|ccc|ccc@{}}
\toprule
{\textbf{Method}  (Percentage = sampling fraction) } & \multicolumn{3}{c}{\textbf{COVID-19 Radiography Database}}                                                                                                                    & \multicolumn{3}{c}{\textbf{APTOS 2019 Blindness Detection}}                                                                                                                                           & \multicolumn{3}{c}{\textbf{OCTMNIST Dataset}}                                                   \\ \cmidrule(l){2-10} 
                                   & \begin{tabular}[c]{@{}c@{}}Normal\\ X-ray\end{tabular} & \begin{tabular}[c]{@{}c@{}}COVID\\ X-ray\end{tabular} & \begin{tabular}[c]{@{}c@{}}Lung Opacity\\ X-ray\end{tabular} & \begin{tabular}[c]{@{}c@{}}Normal\\ Retina\end{tabular} & \begin{tabular}[c]{@{}c@{}}Mild Diabetic\\ Retinopathy\end{tabular} & \begin{tabular}[c]{@{}c@{}}Severe Diabetic\\ Retinopathy\end{tabular} & \begin{tabular}[c]{@{}c@{}}Normal\\ Retina\end{tabular} & DME           & Drusen        \\ \cmidrule(r){1-10}
\textbf{Vanilla FedAvg (Full)}            & \textbf{\textcolor{red}{0.914 $\pm$ 0.007}}                                          & \textbf{\textcolor{orange}{0.924 $\pm$ 0.005}}                                         & \textbf{\textcolor{red}{0.898 $\pm$ 0.007} }                                               & \textbf{\textcolor{red}{0.968 $\pm$ 0.023}}                                           & \textbf{\textcolor{red}{0.927 $\pm$ 0.019} }                                                      & \textbf{\textcolor{red}{0.853 $\pm$ 0.004} }                                                         & \textbf{\textcolor{red}{0.908 $\pm$ 0.026} }                                          & 0.837 $\pm$ 0.103 & \textbf{\textcolor{magenta}{0.855 $\pm$ 0.092} } \\
\textbf{$\baseline$(Full)}            & \textbf{\textcolor{red}{0.953 $\pm$ 0.006}}                                          & \textbf{\textcolor{orange}{0.938 $\pm$ 0.004}}                                         & \textbf{\textcolor{red}{0.902 $\pm$ 0.011} }                                               & \textbf{\textcolor{red}{0.951 $\pm$ 0.057}}                                           & \textbf{\textcolor{red}{0.941 $\pm$ 0.052} }                                                      & \textbf{\textcolor{red}{0.911 $\pm$ 0.028} }                                                         & \textbf{\textcolor{red}{0.926 $\pm$ 0.013} }                                          & 0.851 $\pm$ 0.021 & \textbf{\textcolor{magenta}{0.874 $\pm$ 0.012} } \\

\textbf{Independent Learning (Full)}      & \textbf{\textcolor{magenta}{0.898 $\pm$ 0.001} }                                         & \textbf{\textcolor{magenta}{0.869 $\pm$ 0.002} }                                        & \textbf{\textcolor{orange}{0.884 $\pm$ 0.003} }                                               & \textbf{\textcolor{magenta}{0.945 $\pm$ 0.025} }                                           & 0.877 $\pm$ 0.049                                                       & 0.830 $\pm$ 0.053                                                         & \textbf{\textcolor{orange}{0.890 $\pm$ 0.073} }                                           & 0.798 $\pm$ 0.076 & \textbf{\textcolor{red}{0.890 $\pm$ 0.041} } \\
\textbf{RandomSub FedAvg (50\%)}   & 0.892 $\pm$ 0.024                                          & 0.670 $\pm$ 0.059                                         & 0.583 $\pm$ 0.033                                                & 0.918 $\pm$ 0.047                                           & 0.835 $\pm$ 0.091                                                       & 0.832 $\pm$ 0.021                                                         & 0.811 $\pm$ 0.070                                           & 0.753 $\pm$ 0.089 & 0.805 $\pm$ 0.068 \\
\textbf{LogDet FedAvg (50\%)}      & 0.887 $\pm$ 0.046                                          & 0.838 $\pm$ 0.086                                         & 0.810 $\pm$ 0.062                                                & 0.918 $\pm$ 0.027                                           & \textbf{\textcolor{magenta}{0.885 $\pm$ 0.082} }                                                       & \textbf{\textcolor{magenta}{0.850 $\pm$ 0.057} }                                                         & \textbf{\textcolor{magenta}{0.842 $\pm$ 0.046} }                                           & \textbf{\textcolor{red}{0.897 $\pm$ 0.039}} & 0.845 $\pm$ 0.068 \\
\textbf{DispSum FedAvg (50\%)}     & \textbf{\textcolor{orange}{0.907 $\pm$ 0.015}}                                          & \textbf{\textcolor{red}{0.925 $\pm$ 0.049}}                                        & \textbf{\textcolor{magenta}{0.812 $\pm$ 0.086} }                                               & \textbf{\textcolor{magenta}{0.945 $\pm$ 0.043} }                                           & \textbf{\textcolor{orange}{0.890 $\pm$ 0.095} }                                                       & \textbf{\textcolor{orange}{0.852 $\pm$ 0.061}}                                                         & 0.834 $\pm$ 0.044                                           & \textbf{\textcolor{magenta}{0.887 $\pm$ 0.082} } & \textbf{\textcolor{orange}{0.863 $\pm$ 0.094}} \\
\textbf{DispMin FedAvg (50\%)}     & 0.866 $\pm$ 0.018                                          & 0.780 $\pm$ 0.045                                         & 0.751 $\pm$ 0.069                                                & \textbf{\textcolor{orange}{0.963 $\pm$ 0.021}}                                           & 0.851 $\pm$ 0.067                                                       & 0.765 $\pm$ 0.033                                                         & 0.831 $\pm$ 0.011                                           & \textbf{\textcolor{orange}{0.892 $\pm$ 0.066}} & 0.835 $\pm$ 0.085 \\ 
\textbf{$\methodprop$ (50\%)}            & \textbf{\textcolor{red}{0.932 $\pm$ 0.003}}                                          & \textbf{\textcolor{orange}{0.919 $\pm$ 0.013}}                                         & \textbf{\textcolor{red}{0.871 $\pm$ 0.025} }                                               & 0.921 $\pm$ 0.016                                          & \textbf{\textcolor{red}{0.894 $\pm$ 0.029} }                                                      & \textbf{\textcolor{red}{0.886 $\pm$ 0.017} }                                                         & \textbf{\textcolor{red}{0.916 $\pm$ 0.042} }                                          & 0.805 $\pm$ 0.008 & \textbf{\textcolor{magenta}{0.816 $\pm$ 0.011} } \\
\bottomrule
\end{tabular}%
\label{table:accuracy-report}
}
\subcaption[]{We report classwise accuracies for the current set of proposed methods for all 3 medical datasets. \textbf{\textcolor{red}{Red}} indicates the highest value in accuracy column-wise (i.e. for a particular class for a dataset across all 9 baselines). Similarly \textbf{\textcolor{orange}{Orange}} and \textbf{\textcolor{magenta}{Magenta}} indicates the 2nd and 3rd best classwise accuracy. \textbf{Colors for Vanilla FedAvg, $\baseline$ , $\methodprop$ are grouped together} to primarily compare against subset selection strategies}
\end{table*}

As observed in Table 2 $\methodprop$ performs better than submodular based approaches on average across Covid-19 and APTOS. This is promising as it indicates model centric subset selection is much more useful in terms of performance than model agnostic subset selection methods.

%\begin{figure}[h!]
 %   \centering
  %  \includegraphics[scale=0.5]{images/accuracyPFedBayes.png.png}
   % \caption{Different Notions of Accuracy for Coreset Personalized Federated Learning}
%\label{fig:coresetAccMetric}
%\end{figure}

%\begin{figure}[h!]
  %  \centering
   % \includegraphics[scale=0.35]{images/trainingLoss_compare.png}
    %\caption{Global Training Loss comparison between personal federated learning and coreset federated learning : Coreset based learning has a faster convergence rate as evident from the plot}
    %\label{fig:convergence}
%\end{figure}

%\begin{figure}[h!]
 %   \centering
  %  \includegraphics[scale=0.35]{images/accuracy_compare.png}
   % \caption{Training Accuracy across different settings of coreset size percentage w.r.t original personal federated learning}
   % \label{fig:trainingAccuracy}
%\end{figure}

\section{Conclusion \& Future work}
In this work we proposed several novel objective formulations that draw and synthesize from previous works of two separate domains: coreset optimization and federated learning. Through our extensive experimentations, our proposed method showcases significant gains over traditional federated learning approaches and submodularity based optimization functions followed by Federated Learning. We also showcased through theoretical analysis, how the average generalization error is minimax optimal upto logarithm bounds and how that estimation and approximation error compares against $\baseline$. In future we want to look into how client-wise data distribution affects the current scheme and how we can make the model more robust towards adversarial attacks from (say) skewed data distribution over client-side model parameters. Also, the interplay of how such coreset weights can affect model updates in a privacy-preserving manner is somewhat interesting to explore further.

\bibliography{ref}
\bibliographystyle{iclr2024_conference}

\input{supplemental}

\end{document}

%% file: macros.tex
\def \papertitle{Bayesian Coreset Optimization for Personalized Federated Learning}
\newcommand*{\at}[2][]{#1|_{#2}}

\newtheorem{definition}{Definition}
\newtheorem{proposition}{Proposition}

\newcommand{\RomanNumeralCaps}[1]
    {\MakeUppercase{\romannumeral #1}}

\newtheorem{lemma}{Lemma}

\newtheorem{theorem}{Theorem}
\newtheorem{thm}{Theorem}

\newcommand{\method}[1]{CoreSetPFedBayes}

\newcommand{\methodprop}{\textsc{CoreSet-PFedBayes}}
\newcommand{\baseline}{\textsc{PFedBayes}}
\newcommand{\randomsub}{\textsc{RandomSubset}}

\DeclareMathAlphabet{\pazocal}{OMS}{zplm}{m}{n}

\xpatchcmd{\proof}{\topsep6\p@\@plus6\p@\relax}{}{}{}

\def\BibTeX{{\rm B\kern-.05em{\sc i\kern-.025em b}\kern-.08em
    T\kern-.1667em\lower.7ex\hbox{E}\kern-.125emX}}

%% file: PFedArch_V1.pdf_tex
%% Creator: Inkscape 1.4.2 (ebf0e940, 2025-05-08), www.inkscape.org
%% Accompanies image file 'PFedArch_V1.pdf' (pdf, eps, ps)
\begingroup%
  \makeatletter%
  \providecommand\color[2][]{%
    \errmessage{(Inkscape) Color is used for the text in Inkscape, but the package 'color.sty' is not loaded}%
    \renewcommand\color[2][]{}%
  }%
  \providecommand\transparent[1]{%
    \errmessage{(Inkscape) Transparency is used (non-zero) for the text in Inkscape, but the package 'transparent.sty' is not loaded}%
    \renewcommand\transparent[1]{}%
  }%
  \providecommand\rotatebox[2]{#2}%
  \newcommand*\fsize{\dimexpr\f@size pt\relax}%
  \newcommand*\lineheight[1]{\fontsize{\fsize}{#1\fsize}\selectfont}%
  \ifx\svgwidth\undefined%
    \setlength{\unitlength}{720bp}%
    \ifx\svgscale\undefined%
      \relax%
    \else%
      \setlength{\unitlength}{\unitlength * \real{\svgscale}}%
    \fi%
  \else%
    \setlength{\unitlength}{\svgwidth}%
  \fi%
  \global\let\svgwidth\undefined%
  \global\let\svgscale\undefined%
  \makeatother%
  \begin{picture}(1,0.5625)%
    \lineheight{1}%
    \setlength\tabcolsep{0pt}%
    \put(0,0){\includegraphics[width=\unitlength,page=1]{PFedArch_V1.pdf}}%
  \end{picture}%
\endgroup%

%% file: supplemental.tex
\newpage
\onecolumn
\allowdisplaybreaks
\par\noindent\rule{\textwidth}{1pt}
\begin{center}
\large\textbf{Supplementary Material: 
\papertitle}
\end{center}
\par\noindent\rule{\textwidth}{0.4pt}

In this supplementary material we discuss extensively on the proofs. involved for the theoretical analysis for $\methodprop$ along with more fine-grained experimental details and corresponding baselines.

\section{Proofs}
\label{supplementary:proofs}
Here we discuss the proofs involved with particular propositions and theorems specified in the Theoretical Contributions of this paper.
Utilising the assumptions taken in \cite{zhang2022personalized}, \cite{polson2018posterior} we consider the analysis for equal-width Bayesian Neural network.

\textbf{Assumption 1}: The widths of the neural network are equal width i.e. $s_i = M$.

\textbf{Assumption 2}: Each individual client $i \in [N]$ has equal coreset size of samples $n_k < n$.

\textbf{Assumption 3}: Parameters $s_0, n$ (\textit{total client dataset size}) $n_k$ (\textit{ coreset client dataset size}), $M$, $L$ (number of DNN layers as per Section \ref{Sec:Notations}) are large enough such that the sequence $\sigma_n^2$ is bounded as follows
 
$$
\sigma_n^2=\frac{T}{8 n} A \leq \Omega^2,
$$
where $\tau=\Omega M$ and
$$
\begin{aligned}
& A=\log ^{-1}\left(3 s_0 M\right) \cdot(2 \tau)^{-2(L+1)} \\
& {\left[\left(s_0+1+\frac{1}{\tau-1}\right)^2+\frac{1}{(2 \tau)^2-1}+\frac{2}{(2 \tau-1)^2}\right]^{-1} .}
\end{aligned}
$$

Here $T$ indicates the total number of parameters as defined in Section \ref{Sec:Notations}

Similarly, utilising the coreset regime, we have the following:
$$
\sigma_{n_k}^2=\frac{T}{8 n_k} A \leq \Omega^2,
$$

Since $n_k << n$, hence $\sigma^2_{n_k} >> \sigma^2_{n}$

\textbf{Assumption 4}\label{Assumption:3}: We consider 1-Lipschitz continuous activation function $\sigma(\bullet)$

We also define here a few terms as defined in \citep{zhang2022personalized} which would be useful for our following proof proposals as well.

\begin{definition}\label{def: Prelims_original}
Preliminaries and Definitions required for theoretical proofs under $\baseline$
\begin{align*}
& d^2(\mathcal{P}_{\boldsymbol{\theta}}^i, \mathcal{P}^i)=\mathbb{E}_{X^i}(1-e^{-\frac{\bigl[f^i_{\boldsymbol{\theta}}(X^i) - f^i(X^i)\bigr]^2}{8\sigma^2_{\epsilon}}}) \\
& r_n=((L+1) T / n) \log M+(T / n) \log \left(s_0 \sqrt{n / T}\right) \\
& \xi_n^i=\inf _{\boldsymbol{\theta} \in \boldsymbol{\Theta}(L, \boldsymbol{\mathcal{S}}),\|\boldsymbol{\theta}\|_{\infty} \leq \Omega }\left\|f_{\boldsymbol{\theta}}^i-f^i\right\|_{\infty}^2, \\
& \varepsilon_n=n^{-\frac{1}{2}} \sqrt{(L+1) T \log M+T \log \left(s_0 \sqrt{n / T}\right)} \log ^\delta(n) = \sqrt{r_n}\log ^\delta(n),
\end{align*}
\text{where} $\delta > 1$
\end{definition}

Here $r_n$ indicates the variational error incurred due to the Bayesian approximation to the true posterior distribution in Equation \ref{eq:originalPFedBayes} and $\xi_n^i$ indicates the approximation error incurred during regression w.r.t the actual function to be learnt.

Similarly for the coreset size $n_k$ we define the following:

\begin{definition}\label{def:prelims_coreset_weight}
    Preliminaries and Definitions required for theoretical proofs under $\methodprop$
\begin{align*}
    & d^2(\mathcal{P}_{\boldsymbol{\theta},\boldsymbol{w}}^i, \mathcal{P}^i)=\mathbb{E}_{X^i}(1-e^{-\frac{\bigl[f^i_{\boldsymbol{\theta},\boldsymbol{w}}(X^i) - f^i(X^i)\bigr]^2}{8\sigma^2_{\epsilon}}}) \\
    & \xi_{n_k}^i=\inf _{\boldsymbol{\theta} \in \boldsymbol{\Theta}(L, \boldsymbol{\mathcal{S}}),\|\boldsymbol{\theta}\|_{\infty} \leq \Omega }\left\|f_{\boldsymbol{\theta},\boldsymbol{w}}^i-f^i\right\|_{\infty}^2 \\
    & r_{n_k}=((L+1) T / n_k) \log M+(T / n_k) \log \left(s_0 \sqrt{n_k / T}\right) \\
    & \varepsilon_{n_k}={n_k}^{-\frac{1}{2}} \sqrt{(L+1) T \log M+T \log \left(s_0 \sqrt{n_k / T}\right)} \log ^\delta(n_k) = \sqrt{r_{n_k}}\log ^\delta(n_k)
    & 
\end{align*}
\end{definition}

\begin{lemma}\label{ref: lemma: hellinger symmetry}
    The Hellinger Distance from Definition \ref{def:HellingerDistance} is symmetrical in its arguments $\mathcal{P}^i_{\boldsymbol\theta}$ and $\mathcal{P}^i$.
\end{lemma}
\begin{proof}

    It is easy to show that,
\begin{align}
d^2(\mathcal{P}_{\boldsymbol{\theta}}^i, \mathcal{P}^i) & = \mathbb{E}_{X^i}(1-e^{-\frac{[f^i_{\boldsymbol\theta}(X^i) - f^i(X^i)]^2}{8\sigma^2_{\epsilon}}})\\
&= \mathbb{E}_{X^i}(1-e^{-\frac{[f^i(X^i) - f^i_{\boldsymbol\theta}(X^i)]^2}{8\sigma^2_{\epsilon}}})\\
& = d^2(\mathcal{P}^i, \mathcal{P}_{\boldsymbol{\theta}}^i)
\end{align}\label{proof:HellingeDistanceSymmetry}
\end{proof}

\subsection*{\textbf{Proof. of Theorem }\ref{Theorem:1}}

\begin{thm}\label{thm1:Supplementary}

The difference in the upper bound incurred in the overall generalization error of $\methodprop$ as compared w.r.t that of $\baseline$ is always upper bounded by a closed form positive function that depends on the coreset weights and coreset size- $\boldsymbol{\Im}(\boldsymbol{w}, n_k)$. generalization error in the original full data setup 
 \begin{align*}
    \left[\frac{1}{N}\sum_{i=1}^N\int_{\Theta}d^2(\mathcal{P}^i_{\boldsymbol{\theta}}, \mathcal{P}^i)\hat{q^i}(\boldsymbol{\theta})d\boldsymbol{\theta}\right]_{u.b.}
   -  \left[\frac{1}{N}\sum_{i=1}^N\int_{\Theta}d^2(\mathcal{P}^i_{\boldsymbol{\theta}, w}, \mathcal{P}^i)\hat{q^i}(\boldsymbol{\theta};\boldsymbol{w})d\boldsymbol{\theta}\right]_{u.b.} \leq \boldsymbol{\Im}(\boldsymbol{w}, n_k)
    \end{align*}
\end{thm}

\begin{proof}

Let us define $\log \eta(P_{\boldsymbol\theta}^i, P^i) = l_n(P_{\boldsymbol\theta}^i, P^i)/\zeta + nd^2(P_{\boldsymbol{\theta}}^i, P^i) $.

Using Theorem 3.1 of \cite{pati2018statistical}
with probability at most $e^{-Cn_k\varepsilon^2_{n_k}}$ , where $C$ is a constant, with high probability for $\methodprop$ we have 

\begin{equation}
    \int_{\Theta}\eta(\mathcal{P}^i_{\boldsymbol{\theta}, \boldsymbol{w}}, \mathcal{P}^i)z^{*}(\boldsymbol{\theta})d\boldsymbol{\theta} \leq e^{C{n_k}\varepsilon^2_{n_k}}
\end{equation}

Similarly with high probability at most $e^{-Cn\varepsilon^2_n}$ for the vanilla $\baseline$

\begin{equation}
    \int_{\Theta}\eta(\mathcal{P}^i_{\boldsymbol{\theta}}, \mathcal{P}^i)z^{*}(\boldsymbol{\theta})d\boldsymbol{\theta} \leq e^{Cn\varepsilon^2_n}
\end{equation}

\vspace{4pt}

Using Lemma A.1 from \cite{zhang2022personalized} we know that
for any probability measure $\mu$ and any measurable function $h$ with $e^h \in L_1(\mu)$,
$$
\log \int e^{h(\eta)} \mu(d \eta)=\sup _\rho\left[\int h(\eta) \rho(d \eta)-\mathbb{D}_{KL}(\rho \| \mu)\right]
$$

Further, we let $l_n\left(P^i, P_{\boldsymbol{\theta}}^i\right)$ is the log-likelihood ratio of $P^i$ and $P^i_{\boldsymbol\theta}$
$$
l_n\left(P^i, P_{\theta}^i\right)=\log \frac{\mathcal{P}^i\left(\boldsymbol{D}^i\right)}{\mathcal{P}^i_{\theta}\left(\boldsymbol{D}^i\right)} .
$$

Hence, 

\begin{align*}
nd^2(P_{\boldsymbol{\theta}}^i, P^i)= l_n(P_{\boldsymbol\theta}^i, P^i)/\zeta -\log \eta(P_{\boldsymbol\theta}^i, P^i) \\
= l_n(P^i,P_{\boldsymbol\theta}^i)/\zeta -\log \eta( P^i,P_{\theta}^i)
\tag*{since $d^2(P_{\boldsymbol\theta}^i,P^i) = d^2(P^i,P_{\boldsymbol\theta}^i)$ from Lemma \ref{ref: lemma: hellinger symmetry}} 
\end{align*}
This follows from \ref{proof:HellingeDistanceSymmetry}

Similarly, for the weighted likelihood based Hellinger Distance,

\begin{equation}
n_kd^2(P_{\boldsymbol\theta,\boldsymbol{w}}^i, P^i)= l_n(P_{\boldsymbol{\theta},\boldsymbol{w}}^i, P^i)/\zeta -\log \eta(P_{\boldsymbol{\theta},\boldsymbol{w}}^i, P^i)
\label{eq:expansionHDistance}
\end{equation}

By using Lemma A.1 with $h(\eta)=\log \eta\left(P_{\boldsymbol{\theta}}^i, P^i\right), \mu=z^{\star}(\boldsymbol{\theta})$ and $\rho=\hat{q}^i(\boldsymbol{\theta})$, we obtain
$$
\begin{aligned}
\int_{\Theta} d^2\left(P_{\boldsymbol{\theta}}^i, P^i\right) \hat{q}^i(\boldsymbol{\theta}) d \boldsymbol{\theta} & \leq \frac{1}{n}\left[\frac{1}{\zeta} \int_{\Theta} l_n\left(P^i, P_{\boldsymbol{\theta}}^i\right) \hat{q}^i(\boldsymbol{\theta}) d \boldsymbol{\theta}+\mathbb{D}_{KL}\left(\hat{q}^i(\boldsymbol{\theta}) \| z^{\star}(\boldsymbol{\theta})\right)+\log \int_{\Theta} \eta\left(P_{\boldsymbol{\theta}}^i, P^i\right) z^{\star}(\boldsymbol{\theta}) d \boldsymbol{\theta}\right] \\
& \leq \frac{1}{n}\left[\frac{1}{\zeta} \int_{\Theta} l_n\left(P^i, P_{\boldsymbol{\theta}}^i\right) \hat{q}^i(\boldsymbol{\theta}) d \boldsymbol{\theta}+\mathbb{D}_{KL}\left(\hat{q}^i(\boldsymbol{\theta}) \| z^{\star}(\boldsymbol{\theta})\right)\right]+C \varepsilon_n^2
\end{aligned}
$$

$$
\begin{aligned}
\int_{\Theta}d^2(\mathcal{P}^i_{\boldsymbol\theta, \boldsymbol{w}}, \mathcal{P}^i)\hat{q^i(\boldsymbol{\theta},\boldsymbol{w})}d\boldsymbol{\theta} \leq \frac{1}{n_k}\left[\frac{1}{\zeta} \int_{\Theta} l_n\left(P^i, P_{\boldsymbol{\theta},\boldsymbol{w}}^i\right) \hat{q^i(\boldsymbol{\theta};\boldsymbol{w})} d \boldsymbol{\theta}+\mathbb{D}_{KL}\left(\hat{q^i(\boldsymbol{\theta};\boldsymbol{w})} \| z^{\star}(\boldsymbol{\theta})\right)\right]+C \varepsilon^2_{n_k}
\end{aligned}
$$

Utilising analysis under Supplementary in \citep{bai2020efficient}, there exists an upper bound for the term

\begin{align}\label{ineq: 2nd term}
\int_{\Theta} l_n\left(P^i, P_{\boldsymbol{\theta}}^i\right) \hat{q}^i(\boldsymbol{\theta}) d \boldsymbol{\theta} \leq C^{''}(nr_n + n\xi_{n}^i) 
\end{align}

Lemma 2 from \citep{zhang2022personalized} provides the upper bound for the KL divergence term

\begin{align}\label{ineq: 1st term}
    \mathbb{D}_{KL}\left(\hat{q}^i(\boldsymbol{\theta}) \| z^{\star}(\boldsymbol{\theta})\right) \leq  C^{'}(nr_n)
\end{align}

Therefore we can write the following expression that captures the weighted Hellinger distance displacement given in our coreset framework $\methodprop$ as compared to $\baseline$
$$
\begin{aligned}
\frac{1}{N}\sum_{i=1}^N\int_{\Theta}d^2(\mathcal{P}^i_{\boldsymbol\theta}, \mathcal{P}^i)\hat{q^i(\boldsymbol{\theta})}d\boldsymbol{\theta} - \frac{1}{N}\sum_{i=1}^N\int_{\Theta}d^2(\mathcal{P}^i_{\boldsymbol{\theta}, \boldsymbol{w}}, \mathcal{P}^i)\hat{q^i(\boldsymbol{\theta},\boldsymbol{w})}d\boldsymbol{\theta}\\ 
\leq \frac{1}{N}\sum_{i=1}^N\frac{1}{n}\left[\frac{1}{\zeta} \int_{\Theta} l_n\left(P^i, P_{\boldsymbol{\theta}}^i\right) \hat{q}^i(\boldsymbol{\theta}) d \boldsymbol{\theta}+\mathbb{D}_{KL}\left(\hat{q}^i(\boldsymbol{\theta}) \| z^{\star}(\boldsymbol{\theta})\right)\right]+C \varepsilon_n^2 - \\
\frac{1}{N}\sum_{i=1}^N\frac{1}{n_k}\left[\frac{1}{\zeta} \int_{\Theta} l_n\left(P^i, P_{\boldsymbol{\theta},w}^i\right) \hat{q^i(\boldsymbol{\theta};\boldsymbol{w})} d \boldsymbol{\theta}+\mathbb{D}_{KL}\left(\hat{q^i(\boldsymbol{\theta};\boldsymbol{w})} \| z^{\star}(\boldsymbol{\theta})\right)\right] - C \varepsilon_{n_k}^2\\
\text{Using Eq:(\ref{ineq: 1st term}) and Eq:(\ref{ineq: 2nd term})} \\
\leq C \varepsilon_n^2 - C \varepsilon_{n_k}^2  + 
n \biggl(C^{'}\zeta r_n + \frac{C^{''}}{N}\sum_{i=1}^N \xi_n^i\biggr) - n_k \biggl(C^{'}\zeta r_{n_k} + \frac{C^{''}}{N}\sum_{i=1}^N \xi_{n_k}^i\biggr) \\
\leq C(\varepsilon_n^2 - \varepsilon_{n_k}^2) + \zeta C^{'}(n r_n - n_k r_{n_k}) + \frac{C^{''}}{N}\sum_{i=1}^N (n\xi_n^i - n_k\xi_{n_k}^i)
\\
= C\underbrace{(\varepsilon_n^2 - \varepsilon_{n_k}^2)}_{\substack{\text{Estimation error} \\ \text{Type } \text{\RomanNumeralCaps{1}} \text{\space \textbf{Drift}}}} + \zeta C^{'}\underbrace{(n r_n - n_k r_{n_k})}_{\substack{\text{Estimation error} \\ \text{Type } \text{\RomanNumeralCaps{2}} \text{\space \textbf{Drift}}}} + \frac{C^{''}}{N}\underbrace{\sum_{i=1}^N (n\xi_n^i - n_k\xi_{n_k}^i)}_{\text{Approximation Error \textbf{Drift}}}\\
= \underbrace{\boldsymbol{\Im}(\boldsymbol{w}, n_k)}_{\geq 0}
\end{aligned}
$$
%\pc{\tag*{from Equations. (\ref{eq:original-client-opt}) and (\ref{eq:weighted_client-opt}) respectively}\\}

Where $\boldsymbol{\Im}(\boldsymbol{w}, n_k) =  C(\varepsilon_n^2 - \varepsilon_{n_k}^2) + \zeta C^{'}(n r_n - n_k r_{n_k}) + \frac{C^{''}}{N}\sum_{i=1}^N (n\xi_n^i - n_k\xi_{n_k}^i)$
where each of the coefficients of the closed form function are constants related to $s_0, \beta, \boldsymbol\Lambda, L, M , \zeta$ and $n_k$

Using Lemma \ref{ref:Lemma 2}, \ref{ref: Lemma 3} and with suitable assumptions on the Approximation drift error such that we see that each of the individual error terms are positive, there by indicating $\boldsymbol{\Im}(\boldsymbol{w}, n_k) \geq 0$
\end{proof}

\begin{lemma}\label{ref:Lemma 2} The \text{Estimation error} \text{Type } \text{\RomanNumeralCaps{2}} \text{\space \textbf{Drift}} is a positive quantity i.e.
    $n r_n > n_k r_{n_k}$.
\end{lemma}

\begin{proof}
    By Definition, 
\[r_n=((L+1) T / n) \log M+(T / n) \log \left(s_0 \sqrt{n / T}\right)\]

and 

\[r_{n_k}=((L+1) T / n_k) \log M+(T / n_k) \log \left(s_0 \sqrt{n_k / T}\right)\]

Hence 

\begin{align*}
    \frac{r_n}{r_{n_k}} = \frac{n_k}{n}\times\frac{((L+1) T) \log M+(T) \log \left(s_0 \sqrt{n / T}\right)}{((L+1) T) \log M+(T ) \log \left(s_0 \sqrt{n_k / T}\right)} 
\end{align*}

\begin{align*}
    \frac{r_n}{r_{n_k}} = \frac{n_k}{n}\times\frac{(L+1) \log M+\log \left(s_0 /\sqrt{T}\right) + \log \left(\sqrt{n}\right)}{(L+1)\log M+ \log \left(s_0 /\sqrt{T}\right) + \log \left(\sqrt{n_k}\right)} 
\end{align*}
Considering $(L+1)\log M + \log \left(s_0 /\sqrt{T}\right)$ as a constant $\mathfrak{G}$ we have

\begin{align*}
    \frac{r_n}{r_{n_k}} = \frac{n_k}{n}\times \frac{\mathfrak{G} + \log(\sqrt{n})}{\mathfrak{G} + \log(\sqrt{n_k})}
\end{align*}

Thus,
\begin{align*}
    \frac{n r_n}{n_k r_{n_k}} = \frac{\mathfrak{G} + \log(\sqrt{n})}{\mathfrak{G} + \log(\sqrt{n_k})}
\end{align*}

It is clear since $\log(\bullet)$ is an increasing function for $n > n_k$ we have $n r_n > n_k r_{n_k}$.
\end{proof}

\begin{lemma}\label{ref: Lemma 3}
The \text{Estimation error} \text{Type } \text{\RomanNumeralCaps{1}} \text{\space \textbf{Drift}} is a positive quantity i.e.    $\varepsilon_n^2 > \varepsilon_{n_k}^2$.

\end{lemma}
\begin{proof}
    From the definition under Assumption 3 
    \[\varepsilon_{n_k}={n_k}^{-\frac{1}{2}} \sqrt{(L+1) T \log M+T \log \left(s_0 \sqrt{n_k / T}\right)} \log ^\delta(n_k) = \sqrt{r_{n_k}}\log ^\delta(n_k)\]

    Hence $\varepsilon_{n_k}^2 = r_{n_k}\log^{2\delta}(n_k)$.
    Similarly, $\varepsilon_{n}^2 = r_{n}\log^{2\delta}(n)$

\begin{align*}
\frac{\varepsilon_{n}^2}{\varepsilon_{n_k}^2} = \frac{r_{n}\log^{2\delta}(n)}{r_{n_k}\log^{2\delta}(n_k)} = \frac{n r_{n}\frac{\log^{2\delta}(n)}{n}}{n_k r_{n_k}\frac{\log^{2\delta}(n_k)}{n_k}} 
\end{align*}
From Lemma \ref{ref:Lemma 2}  we know that $nr_n > n_k r_{n_k}$, hence

\begin{align*}
  \frac{\varepsilon_{n}^2}{\varepsilon_{n_k}^2}  > \frac{\frac{\log^{2\delta}(n)}{n}}{\frac{\log^{2\delta}(n_k)}{n_k}}
> 1
\end{align*}

This follows due to the increasing nature of the function.

\end{proof}

\subsection*{\textbf{Proof. of Theorem }\ref{Theorem:2}}

\begin{thm}\label{ref:theorem2:Supplementary}
The convergence rate of the generalization error under $L^2$ norm of $\methodprop$ is minimax optimal up to a logarithmic term (in order $n_k$) for bounded functions ($\beta$-Hölder-smooth functions) $\{f^i\}^N_{i=1}$,  $\{f^i_{\boldsymbol{\theta}}\}^N_{i=1}$ and $\{f^i_{\boldsymbol{\theta},\boldsymbol{w}}\}^N_{i=1}$ where $C_2$, $C_3$ and $\delta^{\prime}$ are constants and $\boldsymbol\Lambda$ being the intrinsic dimension of each client's data:

\begin{align*}
\frac{C_F}{N} \sum_{i=1}^N \int_{\boldsymbol{\theta}}\left\|f_{\boldsymbol{\theta},\boldsymbol{w}}^i\left(X^i\right)-f^i\left(X^i\right)\right\|_{L^2}^2 \hat{q^i}(\boldsymbol{\theta}; \boldsymbol{w}) d \boldsymbol{\theta} 
\leq C_2 n_k^{-\frac{2 \beta}{2 \beta+\boldsymbol{\Lambda}}} \log ^{2 \delta^{\prime}}(n_k) .
\end{align*}
and 
\begin{align*}
\inf _{\left\{\left\|f_{\boldsymbol{\theta}, \boldsymbol{w}}^i\right\|_{\infty} \leq F\right\}_{i=1}^N\left\{\left\|f^i\right\|_{\infty} \leq F\right\}_{i=1}^N} \frac{C_F}{N} \sum_{i=1}^N 
\int_{\boldsymbol{\theta}}\left\|f_{\boldsymbol{\theta}, \boldsymbol{w}}^i\left(X^i\right)-f^i\left(X^i\right)\right\|_{L^2}^2 \hat{q}^i(\boldsymbol{\theta};\boldsymbol{w}) d \boldsymbol{\theta} \geq C_3 n_k^{-\frac{2 \beta}{2 \beta+\boldsymbol\Lambda}}
\end{align*}

where $n_k$ denotes the coreset size per client dataset and $n$ denotes the original per client dataset size and $\frac{d^2\left(P_{\boldsymbol{\theta}, \boldsymbol{w}}^i, P^i\right)}{\left\|f_{\boldsymbol{\theta},\boldsymbol{w}}^i\left
(X^i\right)-f^i\left(X^i\right)\right\|_{L^2}^2} \geq \frac{1-\exp \left(-\frac{4 F^2}{8 \sigma_\epsilon^2}\right)}{4 F^2} \triangleq C_F$.

\end{thm}

We present the choice of $T$ for a typical class of functions. We already assumed that $\left\{f^{i}\right\}$ are $\beta$-Hölder-smooth functions (Definition 4. \citep{nakada2020adaptive}) and the intrinsic dimension of data is $\boldsymbol\Lambda$.

From our above theorem result from Theorem: \ref{thm1:Supplementary} we say the following:

\begin{align}
  \frac{1}{N}\sum_{i=1}^N\int_{\Theta}d^2(\mathcal{P}^i_{\boldsymbol{\theta}, \boldsymbol{w}}, \mathcal{P}^i)\hat{q}^i(\boldsymbol{\theta};\mathbf{w})d\boldsymbol{\theta} \leq C\varepsilon^2_{n_k} + C'r_{n_k}  + \frac{C^{''}}{N\zeta}\sum_{i=1}^N\xi^i_{n_k}
  \label{ineq:Theorem 1 Inequality}
\end{align}

Utilising Corollary 6 in \citep{nakada2020adaptive}
, the approximation error is upper-bounded as follows

\begin{align*}
\left\|f_{\boldsymbol{\theta},\boldsymbol{w}}^i-f^i\right\|_{\infty} \leq C_0 T^{-\frac{\beta}{\boldsymbol{\Lambda}}}
\end{align*}

where $C_0 > 0$ is a constant related to $s_0, \beta$ and $\boldsymbol\Lambda$

Thus from the above definitions \ref{def: Prelims_original} and \ref{def:prelims_coreset_weight}, we have the following
\begin{align*}
    \xi_n^i, \xi_{n_k}^i \leq C_0 T^{-\frac{2\beta}{\boldsymbol{\Lambda}}} , i = 1, \dots, N
\end{align*}

Utilising the above upper bound in \ref{ineq:Theorem 1 Inequality} and substituting $T=C_1n^{\frac{\boldsymbol\Lambda}{2\beta+\boldsymbol\Lambda}}$, we get

\begin{align*}
  \frac{1}{N}\sum_{i=1}^N\int_{\Theta}d^2(\mathcal{P}^i_{\boldsymbol{\theta}, \boldsymbol{w}}, \mathcal{P}^i)\hat{q}^i(\boldsymbol{\theta};\mathbf{w})d\boldsymbol{\theta} \leq   C\varepsilon^2_{n_k} + C'r_{n_k}  + \frac{C^{''}}{N\zeta}\sum_{i=1}^N C_0 T^{-\frac{2\beta}{\boldsymbol{\Lambda}}}\\
  \leq Cr_{n_k}\log^{2\delta}(n_k) + C'r_{n_k}  + \frac{C^{''}}{N\zeta}\sum_{i=1}^N C_0 T^{-\frac{2\beta}{\boldsymbol{\Lambda}}} \hspace{4pt} \mathbf{\because} \hspace{4pt} \varepsilon_{n_k}^2 = r_{n_k}\log^{2\delta}(n_k) \\
  \leq C_2 n_k^{-\frac{2 \beta}{2 \beta+\boldsymbol\Lambda}} \log ^{2 \delta^{\prime}}(n_k) \bigl[  \text{\space \space substituting \space} T \text{\space in \space} r_{n_k} 
\bigr]
\end{align*}

where $\delta^{\prime} > \delta > 1$, and $C_1, C_2 > 0$ are constants related to $s_0, \beta, \boldsymbol\Lambda, L, M , \zeta$ and $n_k$.

Similar to Theorem 1.1 from \citep{bai2020efficient} and Theorem 1 from \citep{zhang2022personalized}
norm, we can write the following

$$
\begin{gathered}
\frac{C_F}{N} \sum_{i=1}^N \int_{\Theta}\left\|f_{\boldsymbol{\theta},\boldsymbol{w}}^i\left(X^i\right)-f^i\left(X^i\right)\right\|_{L^2}^2 \hat{q}^i(\boldsymbol{\theta},\boldsymbol{w}) d \boldsymbol{\theta} \\
\leq \frac{1}{N} \sum_{i=1}^N \int_{\Theta} d^2\left(P_{\boldsymbol{\theta},\boldsymbol{w}}^i, P^i\right) \hat{q}^i(\boldsymbol{\theta};\boldsymbol{w}) d \boldsymbol{\theta} \\
\leq C_2 n_k^{-\frac{2 \beta}{2 \beta+\boldsymbol\Lambda}} \log ^{2 \delta^{\prime}}(n_k) .
\end{gathered}
$$

Now, using the minimax lower bound under $L^2$ norm in Theorem 8 of \citep{nakada2020adaptive}, we see that for coreset regime the same formulation holds similar to our original setting as shown in \citep{zhang2022personalized}

\begin{align*}
    \inf _{\left\{\left\|f_{\boldsymbol{\theta}, \boldsymbol{w}}^i\right\|_{\infty} \leq F\right\}_{i=1}^N\left\{\left\|f^i\right\|_{\infty} \leq F\right\}_{i=1}^N} \frac{C_F}{N} \sum_{i=1}^N 
\int_{\boldsymbol{\theta}}\left\|f_{\boldsymbol{\theta}, \boldsymbol{w}}^i\left(X^i\right)-f^i\left(X^i\right)\right\|_{L^2}^2 \hat{q}^i(\boldsymbol{\theta};\boldsymbol{w}) d \boldsymbol{\theta} \geq C_3 n_k^{-\frac{2 \beta}{2 \beta+\boldsymbol\Lambda}}
\end{align*}

where $C_3 > 0$ is a constant.

Combining the above two equations, the convergence rate of the generalization error of the coreset weighted objective is minimax optimal upto a logarithmic term for bounded functions $\{f^i_{\boldsymbol{\theta},\boldsymbol{w}}\}^N_{i=1}$ and $\{f^i_{\boldsymbol{\theta}}\}^N_{i=1}$.

\subsection*{\textbf{Proof. of Theorem }\ref{Theorem:3}}

\begin{thm}\label{ref: Theorem deviation lower bound}
  The lower bound (l.b.) incurred for the deviation for the weighted coreset $\methodprop$ (\ref{eq:weighted_client-opt}) generalization error is always higher than the lower bound of that for the original $\baseline$ objective (\ref{eq:originalPFedBayes}) with a delta difference (\textbf{Error I} - \textbf{Error II}) as $\mathcal{O}(n_k^{-\frac{2 \beta}{2 \beta+\boldsymbol\Lambda}})$ 
\begin{align*}
\Bigg[\underbrace{\sum_{i=1}^N\int_{\boldsymbol{\Theta}}\left\|f_{\boldsymbol{\theta}, \boldsymbol{w}}^i\left(X^i\right)-f^i\left(X^i\right)\right\|_{L^2}^2 \hat{q^i}(\boldsymbol{\theta}, \boldsymbol{w}) d \boldsymbol{\theta}}_{\text{Coreset weighted objective Generalization Error (\textbf{Error I})}} \Bigg]_{l.b.} 
> \Bigg[\underbrace{\sum_{i=1}^N\int_{\boldsymbol{\Theta}}\left\|f_{\boldsymbol{\theta}}^i\left(X^i\right)-f^i\left(X^i\right)\right\|_{L^2}^2 \hat{q^i}(\boldsymbol{\theta}) d \boldsymbol{\theta}}_{\text{Vanilla objective Generalization Error (\textbf{Error II})}}\Bigg]_{l.b.}
\end{align*}
\end{thm}

\begin{proof}

As we know $n_k < n$ hence $C_3 n_k^{-\frac{2 \beta}{2 \beta+\boldsymbol{\Lambda}}} >  C_3 n^{-\frac{2 \beta}{2 \beta+\boldsymbol\Lambda}}$ ($\because$ $C_3$ is a constant independent of $n$ or $n_k$), which therefore means that inequality holds in the lower bound (l.b.) of the two expressions (shown by the previous proposition 2).

\begin{align*}
\Bigg[\sum_{i=1}^N\int_{\boldsymbol{\theta}}\left\|f_{\boldsymbol{\theta}, \boldsymbol{w}}^i\left(X^i\right)-f^i\left(X^i\right)\right\|_{L^2}^2 \hat{q}^i(\boldsymbol{\theta}, \boldsymbol{w}) d \boldsymbol{\theta} \Bigg]_{l.b.} 
> \Bigg[\sum_{i=1}^N\int_{\boldsymbol{\theta}}\left\|f_{\boldsymbol{\theta}}^i\left(X^i\right)-f^i\left(X^i\right)\right\|_{L^2}^2 \hat{q}^i(\boldsymbol{\theta}) d \boldsymbol{\theta}\Bigg]_{l.b.}
\end{align*}

Let us denote $\boldsymbol{\Delta}^{l.b}_{deviation}$ as follows

\[\boldsymbol{\Delta}^{l.b}_{deviation} = \Bigg[\sum_{i=1}^N\int_{\boldsymbol{\theta}}\left\|f_{\boldsymbol{\theta}, \boldsymbol{w}}^i\left(X^i\right)-f^i\left(X^i\right)\right\|_{L^2}^2 \hat{q}^i(\boldsymbol{\theta}, \boldsymbol{w}) d \boldsymbol{\theta} \Bigg]_{l.b.} - \Bigg[\sum_{i=1}^N\int_{\boldsymbol{\theta}}\left\|f_{\boldsymbol{\theta}}^i\left(X^i\right)-f^i\left(X^i\right)\right\|_{L^2}^2 \hat{q}^i(\boldsymbol{\theta}) d \boldsymbol{\theta}\Bigg]_{l.b.} \]

And the $\boldsymbol{\Delta}^{l.b.}_{deviation}$ term is given by  $\bigg(C_3 n_k^{-\frac{2 \beta}{2 \beta+\boldsymbol{\Lambda}}} -  C_3 n^{-\frac{2 \beta}{2 \beta+\boldsymbol{\Lambda}}}\bigg) \approx \mathcal{O}(n_k^{-\frac{2 \beta}{2 \beta+\boldsymbol{\Lambda}}})
$.
\end{proof}

\subsection*{\textbf{Proof. of Theorem }\ref{Theorem:4}}
\begin{thm}\label{ref: Supplementary Theorem:4}
The lower bound incurred in the overall generalization error across all $N$ clients of $\methodprop$ is always higher compared to that of the generalization error in the original full data setup 
 \begin{align*}
    \left[\frac{1}{N}\sum_{i=1}^N\int_{\Theta}d^2(\mathcal{P}^i_{\boldsymbol{\theta}, w}, \mathcal{P}^i)\hat{q^i}(\boldsymbol{\theta};\boldsymbol{w})d\boldsymbol{\theta}\right]_{l.b.} \geq 
    \left[\frac{1}{N}\sum_{i=1}^N\int_{\Theta}d^2(\mathcal{P}^i_{\boldsymbol{\theta}}, \mathcal{P}^i)\hat{q^i}(\boldsymbol{\theta})d\boldsymbol{\theta}\right]_{l.b.}
    \end{align*}
\end{thm}

\begin{proof}
It is easy to show since from Theorem \ref{ref: Theorem deviation lower bound}, we know the lower bounds for the individual terms and also since $n > n_k$ holds, hence we can rewrite as follows:
$$
\begin{aligned}
\frac{1}{N}\sum_{i=1}^N\int_{\Theta}d^2(\mathcal{P}^i_{\boldsymbol{\theta}, \boldsymbol{w}}, \mathcal{P}^i)\hat{q^i(\boldsymbol{\theta},\boldsymbol{w})}d\boldsymbol{\theta} - \frac{1}{N}\sum_{i=1}^N\int_{\Theta}d^2(\mathcal{P}^i_{\boldsymbol\theta}, \mathcal{P}^i)\hat{q^i(\boldsymbol{\theta})}d\boldsymbol{\theta} \\
\geq C_3 n_k^{-\frac{2 \beta}{2 \beta+\boldsymbol\Lambda}} - C_3 n^{-\frac{2 \beta}{2 \beta+\boldsymbol\Lambda}} \\
\geq 0  
\end{aligned}
$$

The implication of this proof states that the overall error incurred due to coreset weighted deviation is always more than that of the original deviation which can be measured approximately in order of $n_k$, the coreset sample size.
\end{proof}

\begin{proposition}
    The gradient of the first term in Equation \ref{eq: individual Kl} i.e. \[\nabla_{w}\mathbb{D}_{KL}( \hat{q^{i}}(\boldsymbol{\theta};\mathbf{w}) ||  \hat{q^{i}}(\boldsymbol{\theta}))\]
    is given by the following expression

    \[
\int_{\Theta}\nabla_w \hat{q^{i}}(\boldsymbol{\theta};\mathbf{w})\biggl[ \log \hat{q^{i}}(\boldsymbol{\theta};\mathbf{w}) + 1 - \log \hat{q^i}(\boldsymbol{\theta})\biggr]d\boldsymbol{\theta}
    \]

    where

    $\nabla_w \hat{q^{i}}(\boldsymbol{\theta};\mathbf{w})
 = \frac{\hat{q^{i}}(\boldsymbol{\theta};\boldsymbol{w})}{\varrho^i(\boldsymbol{\theta_{i,m}}; \boldsymbol{w})}g^{'}_m(\boldsymbol{w}) + g_m(\boldsymbol{w})\nabla_w\prod_{k \neq m}^T \varrho^i(\boldsymbol{\theta_{i,k}};\boldsymbol{w})$
   and 
   $    q^i(\boldsymbol{\theta};\boldsymbol{w}) 
    = \prod_{m=1}^T \varrho^i(\boldsymbol{\theta}_{i,m}; \boldsymbol{w})$
\end{proposition}

\begin{proof}

%\pc{w should be curved everywhere}

\begin{equation}
\begin{aligned}
\nabla_{w}\mathbb{D}_{KL}( \hat{q^{i}}(\theta;\mathbf{w}) ||  \hat{q^{i}}(\theta)) \\ 
 = 
\nabla_w \mathrm{E}_{\hat{q^{i}}(\theta;\mathbf{w})}\biggl[\log \hat{q^{i}}(\theta;\mathbf{w}) - \log \hat{q^{i}}(\theta) \biggr] \\
= \nabla_w\biggl[\int_{\Theta}\hat{q^{i}}(\boldsymbol{\theta};\mathbf{w})\log \hat{q^{i}}(\boldsymbol{\theta};\mathbf{w})d\boldsymbol{\theta} - \int_{\Theta}\hat{q^{i}}(\boldsymbol{\theta};\mathbf{w})\log \hat{q^{i}}(\boldsymbol{\theta})d\boldsymbol{\theta}\biggr] \\
= \biggl[\int_{\Theta}\nabla_w\biggl(\hat{q^{i}}(\boldsymbol{\theta};\mathbf{w})\log \hat{q^{i}}(\boldsymbol{\theta};\mathbf{w})\biggr)d\boldsymbol{\theta} - \int_{\Theta}\nabla_w\biggl(\hat{q^{i}}(\boldsymbol{\theta};\mathbf{w})\log \hat{q^{i}}(\boldsymbol{\theta})\biggr)d\boldsymbol{\theta}\biggr] \\
= \biggl[
\int_{\Theta} \biggl(\log
\hat{q^{i}(\boldsymbol{\theta};\mathbf{w})}\nabla_w \hat{q^{i}}(\boldsymbol{\theta};\mathbf{w}) + \nabla_w \hat{q^{i}}(\boldsymbol{\theta};\mathbf{w})\biggr)d\boldsymbol{\theta}
- \int_{\Theta}\log \hat{q^{i}}(\boldsymbol{\theta})\nabla_w\hat{q^{i}}(\boldsymbol{\theta};\mathbf{w})d\boldsymbol{\theta}\biggr]\\
= \int_{\Theta}\nabla_w \hat{q^{i}}(\boldsymbol{\theta};\mathbf{w})\biggl[ \log \hat{q^{i}}(\boldsymbol{\theta};\mathbf{w}) + 1 - \log \hat{q^i}(\boldsymbol{\theta})\biggr]d\boldsymbol{\theta}
\end{aligned}
\label{proof:main-equation}
\end{equation}

In order to compute the gradient $\nabla_w \hat{q^{i}}(\boldsymbol{\theta};\mathbf{w})$, 
the following objective can be utilized.

Let $z^*(\boldsymbol\theta)$ be the optimal variable solution to Equation (\ref{eq:weighted_client-opt}).
\begin{align*}
\nabla_{q^i(\theta)}F^w_i(z^*)\at[\bigg]{\hat{q^{i}(\theta;\mathbf{w})}} = 0 \\
 \implies  \underbrace{\nabla_{q^i(\boldsymbol\theta)}\int_{\Theta}-\log \mathcal{P}_{\theta, w}(\boldsymbol{\mathcal{D}}^i)q^i(\boldsymbol{\theta})d\boldsymbol{\theta}}_\textbf{First Part}\at[\bigg]{\hat{q^{i}(\theta;\mathbf{w})}}
 + \underbrace{\zeta \nabla_{q^i(\boldsymbol\theta)}\mathbb{D}_{KL}(q^{i}(\boldsymbol\theta) || z^*(\boldsymbol\theta))}_\textbf{Second Part}\at[\bigg]{\hat{q^{i}(\theta;\mathbf{w})}} = 0
 \\ 
\end{align*}

For the \textbf{first part},
\begin{equation}
    \begin{aligned}
        \nabla_{q^i(\theta)}\int_{\Theta}-\log \mathcal{P}_{\theta, w}(\boldsymbol{\mathcal{D}}^i)q^i(\boldsymbol{\theta})d\boldsymbol{\theta}\at[\bigg]{\hat{q^{i}(\theta;\mathbf{w})}}\\
        = \int_{\Theta}\nabla_{q^i(\theta)}\biggl[-\log \mathcal{P}_{\theta, w}(\boldsymbol{\mathcal{D}}^i)q^i(\boldsymbol{\theta})\biggr]d\boldsymbol{\theta}\at[\bigg]{\hat{q^{i}(\theta;\mathbf{w})}}\\
    = \underbrace{\int_{\Theta}\biggl[-q^i(\boldsymbol{\theta})\nabla_{q^i(\theta)}\log \mathcal{P}_{\theta, w}(\boldsymbol{\mathcal{D}}^i) + \log \mathcal{P}_{\theta, w}(\boldsymbol{\mathcal{D}}^i)\biggr]d\boldsymbol{\theta}}_\textbf{Modified First part}\at[\bigg]{\hat{q^{i}(\theta;\mathbf{w})}}\\
    \end{aligned}
        \label{proof:firstPart_}
\end{equation}

By the assumption that the distribution $q^i(\boldsymbol{\theta})$ satisfies mean-field decomposition i.e.

\begin{equation}
\begin{aligned}
    q^i(\boldsymbol{\theta}) = \prod_{m=1}^T\mathcal{N}(\theta_{i,m},\sigma^2_n) \\
    = \prod_{m=1}^T \varrho^i(\boldsymbol{\theta}_{i,m})
\end{aligned}  
\end{equation}

Let us denote $\mathcal{M}_w = \mathcal{P}_{\theta,w}(\boldsymbol{\mathcal{D}}^i)$.

Therefore, we extract out the following portion from (\ref{proof:firstPart_}): $\nabla_{q^i(\theta)}\log \mathcal{P}_{\theta, w}(\boldsymbol{\mathcal{D}}^i)$

\paragraph{}

\begin{equation}
    \nabla_{q^i(\theta)}\log \mathcal{P}_{\theta, w}(\boldsymbol{\mathcal{D}}^i) = \nabla_{q^i(\theta)}\log\mathcal{M}_w
\end{equation}

\paragraph{}

We now consider the individual partial differentials here

\begin{equation}
    \frac{\partial} 
   {\partial 
 \varrho^i(\boldsymbol{\theta_{i,m}})}\log \mathcal{M}_w \\
    = \frac{1}{\mathcal{M}_w}\frac{\partial \mathcal{M}_w}{\partial w}\frac{\partial w}{\partial \varrho^i(\boldsymbol{\theta_{i,m}})}
\end{equation}

Thus, we can rewrite (\ref{proof:firstPart_}) from the perspective of individual components of $q^i(\boldsymbol{\theta})$ as follows:

\begin{equation}
\begin{aligned}
    \int_{\Theta}\biggl[-q^i(\boldsymbol{\theta})\frac{1}{\mathcal{M}_w}\frac{\partial \mathcal{M}_w}{\partial w}\frac{\partial w}{\partial \varrho^i(\boldsymbol{\theta_{i,m}})} + \log \mathcal{P}_{\theta, w}(\boldsymbol{\mathcal{D}}^i)\biggr]d\boldsymbol{\theta}\at[\bigg]{\hat{q^{i}(\theta;\mathbf{w})}}\\
    =  \underbrace{\int_{\Theta}\biggl[-q^i(\boldsymbol{\theta})\frac{1}{\mathcal{P}_{\theta,w}(\boldsymbol{\mathcal{D}}^i)}\frac{\partial \mathcal{P}_{\theta,w}(\boldsymbol{\mathcal{D}}^i)}{\partial w}\frac{\partial w}{\partial \varrho^i(\boldsymbol{\theta_{i,m}})} + \log \mathcal{P}_{\theta, w}(\boldsymbol{\mathcal{D}}^i)\biggr]d\boldsymbol{\theta}}_\textbf{Modified First part}\at[\bigg]{\hat{q^{i}(\theta;\mathbf{w})}}
\end{aligned}
\end{equation}

Now, we can rewrite the \textbf{second part} as follows:

\begin{equation}
    \begin{aligned}
         \zeta \nabla_{q^i(\theta)}\mathbb{D}_{KL}(q^{i}(\boldsymbol{\theta}) || z^*(\theta))\at[\bigg]{\hat{q^{i}(\boldsymbol{\theta};\mathbf{w})}}\\
         = \zeta\nabla_{q^i(\theta)}\biggl[\int_{\Theta}q^{i}(\boldsymbol{\theta})\log q^{i}(\boldsymbol{\theta}) - q^{i}(\boldsymbol{\theta})\log(z^*(\boldsymbol{\theta}))d\boldsymbol{\theta}\biggr]\at[\bigg]{\hat{q^{i}(\theta;\mathbf{w})}}\\
         = \zeta\int_{\Theta}\nabla_{q^i(\theta)} \biggl[q^{i}(\boldsymbol{\theta})\log q^{i}(\boldsymbol{\theta}) - q^{i}(\boldsymbol{\theta})\log(z^*(\boldsymbol{\theta}))\biggr]d\boldsymbol{\theta}\at[\bigg]{\hat{q^{i}(\theta;\mathbf{w})}}\\
         =
         \underbrace{\zeta\int_{\Theta}\biggl(\log  q^{i}(\boldsymbol{\theta}) + 1 -\log(z^*(\boldsymbol{\theta}))\biggr)d\boldsymbol{\theta}}_\textbf{Modified Second Part}\at[\bigg]{\hat{q^{i}(\theta;\mathbf{w})}}
    \end{aligned}
\end{equation}
    
\end{proof}

Combining both the first and second part we get

\begin{equation}
    \begin{aligned}
    \int_{\Theta}\biggl[-\hat{q^i(\boldsymbol{\theta}; \boldsymbol{w})}\frac{1}{\mathcal{P}_{\theta,w}(\boldsymbol{\mathcal{D}}^i)}\frac{\partial \mathcal{P}_{\theta,w}(\boldsymbol{\mathcal{D}}^i)}{\partial w}\frac{\partial w}{\partial \varrho^i(\boldsymbol{\theta_{i,m}})} + \log \mathcal{P}_{\theta, w}(\boldsymbol{\mathcal{D}}^i)\biggr]d\boldsymbol{\theta} \\
    +
\zeta\int_{\Theta}\biggl(\log  \hat{q^{i}(\boldsymbol{\theta};\boldsymbol{w})} + 1 -\log(z^*(\boldsymbol{\theta}))\biggr)d\boldsymbol{\theta}  = 0 \\
\implies \zeta\int_{\Theta}\biggl(\log  \hat{q^{i}(\boldsymbol{\theta};\boldsymbol{w})} + 1 -\log(z^*(\boldsymbol{\theta})) + \log \mathcal{P}_{\theta, w}(\boldsymbol{\mathcal{D}}^i)\biggr)d\boldsymbol{\theta} \\  = \int_{\Theta}\biggl(\hat{q^i(\boldsymbol{\theta}; \boldsymbol{w})}\frac{1}{\mathcal{P}_{\theta,w}(\boldsymbol{\mathcal{D}}^i)}\frac{\partial \mathcal{P}_{\theta,w}(\boldsymbol{\mathcal{D}}^i)}{\partial w}\frac{\partial w}{\partial \varrho^i(\boldsymbol{\theta_{i,m}})}\biggr)d\boldsymbol{\theta} \\
\end{aligned}
\label{eq:variationalSolOptimalEq}
\end{equation}

\paragraph{}

Let us assume without loss of generality that each of the individual components of the optimal coreset weighted client distribution $\hat{q^i(\boldsymbol{\theta}; \boldsymbol{w})}$ can be denoted as some function $g(\boldsymbol{w})$. More, specifically,

\begin{equation}
\begin{aligned}
\varrho^i(\boldsymbol{\theta_{i,j}}; \boldsymbol{w}) = g_j(\boldsymbol{w}) \\
        \nabla_w\varrho^i (\boldsymbol{\theta_{i,j}}; \boldsymbol{w}) = g^{'}_j(\boldsymbol{w})
\end{aligned}
\end{equation}

\

Thus we can reuse the above expression to simplify (\ref{eq:variationalSolOptimalEq})

\begin{equation}
    \begin{aligned}
        g^{'}_m(\boldsymbol{w})  = \frac{\int_{\Theta}\biggl(\hat{q^i(\boldsymbol{\theta}; \boldsymbol{w})}\frac{1}{\mathcal{P}_{\theta,w}(\boldsymbol{\mathcal{D}}^i)}\frac{\partial \mathcal{P}_{\theta,w}(\boldsymbol{\mathcal{D}}^i)}{\partial w}\biggr)d\boldsymbol{\theta}}{\zeta\int_{\Theta}\biggl(\log  \hat{q^{i}(\boldsymbol{\theta};\boldsymbol{w})} + 1 -\log(z^*(\boldsymbol{\theta})) + \log \mathcal{P}_{\theta, w}(\boldsymbol{\mathcal{D}}^i)\biggr)d\boldsymbol{\theta}} 
    \end{aligned}
\end{equation}

\paragraph{}

We now go back to utilizing the above derived expression in our main Eq. (\ref{proof:main-equation}) to replace $\nabla_w \hat{q^{i}(\boldsymbol{\theta};\mathbf{w})}$

\begin{equation}
    \begin{aligned}
        \nabla_w \hat{q^{i}(\boldsymbol{\theta};\boldsymbol{w})} \\
        = \nabla_w\prod_{k=1}^T \varrho^i(\boldsymbol{\theta_{i,k}};\boldsymbol{w}) \\
        = \nabla_w \prod_{k=1}^T g_k(\boldsymbol{w}) \\
        = \prod_{k \neq m}^T g_k(\boldsymbol{w})\nabla_w g_m(\boldsymbol{w}) + g_m(\boldsymbol{w})\nabla_w\prod_{k \neq m}^T g_k(\boldsymbol{w}) \\
       = \prod_{k \neq m}^T g_k(\boldsymbol{w})g^{'}_m(\boldsymbol{w}) + g_m(\boldsymbol{w})\nabla_w\prod_{k \neq m}^T g_k(\boldsymbol{w})\\
       = \frac{\hat{q^{i}(\boldsymbol{\theta};\boldsymbol{w})}}{\varrho^i(\boldsymbol{\theta_{i,m}}; \boldsymbol{w})}g^{'}_m(\boldsymbol{w}) + g_m(\boldsymbol{w})\nabla_w\prod_{k \neq m}^T \varrho^i(\boldsymbol{\theta_{i,k}};\boldsymbol{w})\\
    \end{aligned}
\end{equation}

Thus, we now have a closed form solution to computing the gradient of the KL divergence $\mathbb{D}( \hat{q^{i}(\boldsymbol{\theta};\mathbf{w})} ||  \hat{q^{i}(\boldsymbol{\theta})})$ w.r.t the coreset weight parameters $\boldsymbol{w}$.

%\pc{equation reference is missing here}

\begin{proposition}
    The gradient of the second term in Equation \ref{eq:modified-personalized-weights} w.r.t $\boldsymbol{w}$ i.e.
    \[\nabla_w||P_{\theta}(\boldsymbol{\mathcal{D}^{i}}) - P_{\theta, w}(\boldsymbol{\mathcal{D}^i})||^2_{\hat{\pi},2}\]
    
    is given by the following expression

\begin{align*}
-2\mathcal{P}_{\boldsymbol{\Phi}}^T\biggl(\mathcal{P} - \mathcal{P}_{\boldsymbol{\Phi}}\boldsymbol{w}\biggr) 
\end{align*}

where  $\mathcal{P} = \sum_{j=1}^n\hat{g_j}$ and $\mathcal{P}_{\boldsymbol{\Phi}} = [\hat{g_1}, \hat{g_2}, \dots, \hat{g_n}]$

\label{prop:AIHT_simplified}
\end{proposition}

\paragraph{}

\begin{proof}
First, we reformulate the given expression in terms 

\begin{align*}
\left\lVert P_{\theta}(\boldsymbol{\mathcal{D}}^{i}) - P_{\theta, w}(\boldsymbol{\mathcal{D}}^i)\right\rVert^2_{\hat{\pi},2} \\
= \mathrm{E}_{\theta \sim \hat{\pi}}[(P_{\theta}(\boldsymbol{\mathcal{D}}^{i}) - P_{\theta, w}(\boldsymbol{\mathcal{D}}^i))^2] 
\end{align*}

We define $g_j = \mathcal{P}_{\theta}(\boldsymbol{\mathcal{D}}^i_j) - \mathrm{E}_{\theta \sim \hat{\pi}}P_{\theta}(\boldsymbol{\mathcal{D}}^{i}_j)$

As a result the equivalent optimization problem becomes minimizing $ \left\lVert \sum_{j=1}^n g_j - \sum_{j=1}^nw_jg_j\right\rVert^2_{\hat{\pi},2}$

Further, using Monte Carlo approximation, given $S$ samples $\{\theta_j\}_{j=1}^S$, $\theta_j \sim \hat{\pi}$, the $L^2(\hat{\pi})$-norm can be approximated as follows
\begin{align*}
\left\lVert\sum_{j=1}^n\hat{g_j} - \sum_{j=1}^nw_j\hat{g_j}\right\rVert^2_2
\end{align*}

where

$\hat{g_j} = \frac{1}{\sqrt{S}}\bigl[\mathcal{P}_{\theta_1}(\boldsymbol{\mathcal{D}}^i_j) - \Bar{\mathcal{P}(\boldsymbol{\mathcal{D}}^i_j)},\mathcal{P}_{\theta_2}(\boldsymbol{\mathcal{D}}^i_j) - \Bar{\mathcal{P}(\boldsymbol{\mathcal{D}}^i_j)},\dots, \mathcal{P}_{\theta_S}(\boldsymbol{\mathcal{D}}^i_j) - \Bar{\mathcal{P}(\boldsymbol{\mathcal{D}}^i_j)}\bigr]$ and $\Bar{\mathcal{P}(\boldsymbol{\mathcal{D}}^i_j)} = \frac{1}{S}\sum_{k=1}^S\mathcal{P}_{\theta_k}(\boldsymbol{\mathcal{D}}^i_j)$

\paragraph{}

We can write the above problem in matrix notation as follows

\begin{align*}
f(\boldsymbol{w}):= \left\lVert \mathcal{P} - \mathcal{P}_{\boldsymbol{\Phi}} \boldsymbol{w}\right\rVert_2^2
\end{align*}

where  $\mathcal{P} = \sum_{j=1}^n\hat{g_j}$ and $\mathcal{P}_{\boldsymbol{\Phi}} = [\hat{g_1}, \hat{g_2}, \dots, \hat{g_n}]$

Thus we have the gradient w.r.t $\boldsymbol{w}$ as follows:

\begin{equation}
\begin{aligned}
    \nabla_w f(\boldsymbol{w}) = -2\mathcal{P}_{\boldsymbol{\Phi}}^T\biggl(\mathcal{P} - \mathcal{P}_{\boldsymbol{\Phi}}\boldsymbol{w}\biggr)
\end{aligned}
\end{equation}
\end{proof}

\section{Experiments}
All the experiments have been done using the following configuration: Nvidia RTX A4000(16GB) and  Apple M2 Pro 10 cores and 16GB memory.

\subsection{Proposal for a modified objective in Equation \ref{eq:modified-personalized-weights}}

\begin{equation}
\begin{aligned}
 \{\boldsymbol{w_i}^*\} \triangleq \arg \min_{\boldsymbol{w}}\mathbb{D}_{KL}(\hat{q^{i}}(\boldsymbol{\theta}, \boldsymbol{w}) || \hat{q^{i}}(\boldsymbol{\theta}) ) + \left\lVert P_{\boldsymbol{\theta}}(\boldsymbol{\mathcal{D}^{i}}) - P_{\boldsymbol{\theta}, \boldsymbol{w_i}}(\boldsymbol{\mathcal{D}^i})\right\rVert^2_{\hat{\pi},2} 
\hspace{7pt} ||\boldsymbol{w_i}||_{0} \leq k 
\end{aligned}
\end{equation}

We discuss here the utility of our proposed modified client side objective function via an ablation study where we want to gauge the inclusion of the first term in our objective function as just inlcuding the coreset loss.

Through experimental analysis, we find that just including the coreset loss optimization results in early saturation, possibly hinting towards getting stuck in local minima, but however inclusion the KL Divergence loss and forcing the coreset weighted local distribution of the client and the normal local distribution of the client to be similar leads to better stability in the training loss and better convergence.

\begin{figure}[h!]
    \centering
    \includegraphics[scale=0.15]{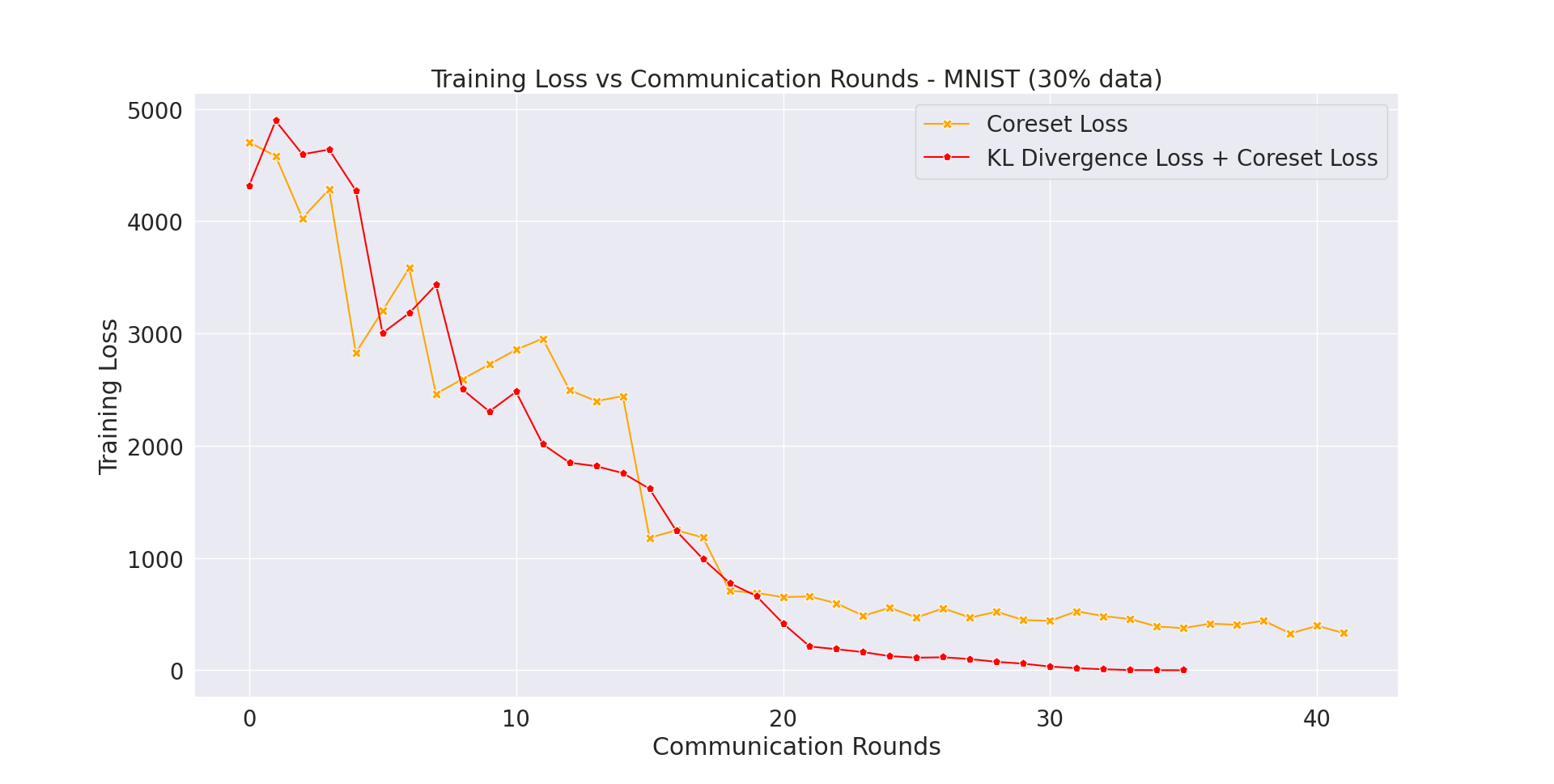}%
        \includegraphics[scale=0.15]{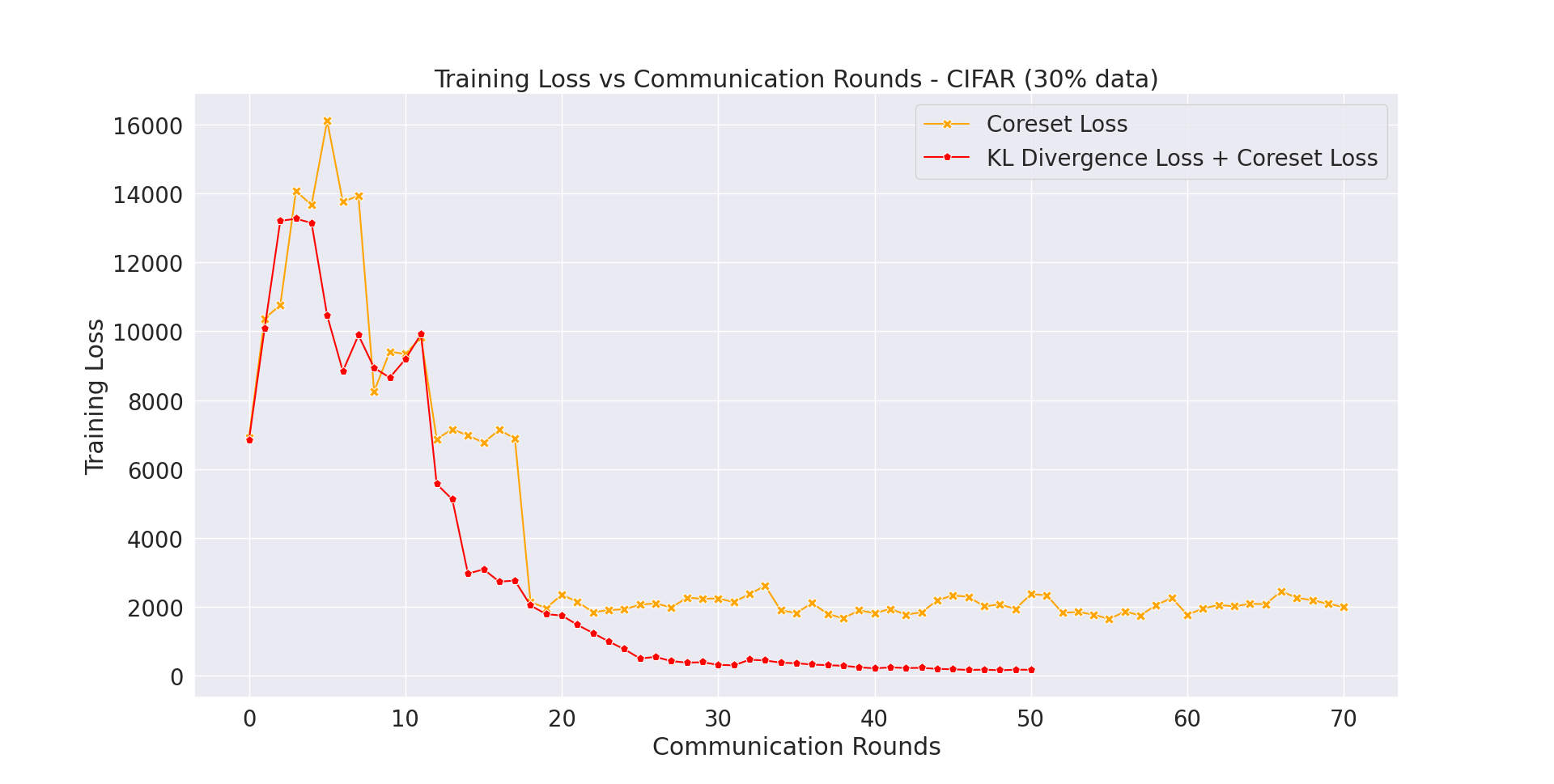}
    \caption{Ablation Study on using KL divergence between two local distribution w.r.t just using coreset weights}
    \label{fig:enter-label}
\end{figure}

\subsection{Communication Complexity Analysis For Different Coreset Sizes}\label{sec:Communication Complexity Analysis For Different Coreset Sizes}

Here we showcase an analysis for different coreset sample size for different datasets and how it affects on the final accuracy and the total number of communication rounds in the Federated Learning setting. This showcases cost-effectiveness of our approach where by using only a small number of communication rounds our proposed approach is able to attain near-optimal performance as per the table below. In addition Fig: \ref{fig:conv_analysis} substantiates the cost-effectiveness of our approach.
    
    \begin{table*}[h!]
    \centering
    \caption{Comparative results of test accuracies across different coreset sample complexity}
    \label{Table : Accuracy Stats}
    \resizebox{\textwidth}{!}{%
    \begin{tabular}{@{}c|cc|cc|cc@{}}
    \toprule
    {\textbf{Method}  (Percentage = sampling fraction) } & \multicolumn{2}{c}{\textbf{MNIST}}                                                                                                                    & \multicolumn{2}{c}{\textbf{FashionMNIST}}                                                                                                                                           & \multicolumn{2}{c}{\textbf{CIFAR}}                                                   \\ \cmidrule(l){2-7} 
                                       & \begin{tabular}[c]{@{}c@{}}Test Accuracy\end{tabular} & \begin{tabular}[c]{@{}c@{}}Communication Rounds\end{tabular} & \begin{tabular}[c]{@{}c@{}}Test Accuracy\end{tabular} & \begin{tabular}[c]{@{}c@{}}Communication Rounds\end{tabular} & \begin{tabular}[c]{@{}c@{}}Test Accuracy\end{tabular} & \begin{tabular}[c]{@{}c@{}}Communication Rounds\end{tabular} \\ \cmidrule(r){1-7}
    \textbf{$\baseline$ (Full)}  & 98.79& 194 & 93.01 & 215 &83.46 & 266
    \\
    \textbf{$\randomsub$ (50\%)}  & 80.2 & 135 & 87.12 & 172 &48.31 &183
    \\
    \textbf{$\methodprop$ (k = 50\%)}  & 92.48 & 98 & 89.55 & 93 & 69.66 & 112
    \\
    \textbf{$\methodprop$ (k = 30\%)}  & 90.17 & 84 & 88.16 & 72 & 59.12 & 70 
    \\
    \textbf{$\methodprop$ (k = 15\%)}  & 88.75 & 62 & 85.15 & 38 & 55.66 & 32
    \\
    \textbf{$\methodprop$ (k = 10\%)}  & 85.43 & 32 & 82.64 & 24 & 48.25 & 16
    \\
    \bottomrule
    \end{tabular}%
    \label{table:communication-complexity-analysis}}
    \subcaption[]{We report test accuracies across different sample complexity for datasets like \textbf{MNIST, CIFAR, FashionMNIST}.
    \textbf{Full indicates training on full dataset and 50\% is on using half the data size after randomly sampling 50\% of the training set.
    }}
    \end{table*}

\begin{figure}[h!]
\centering
\includegraphics[width=.5\textwidth]{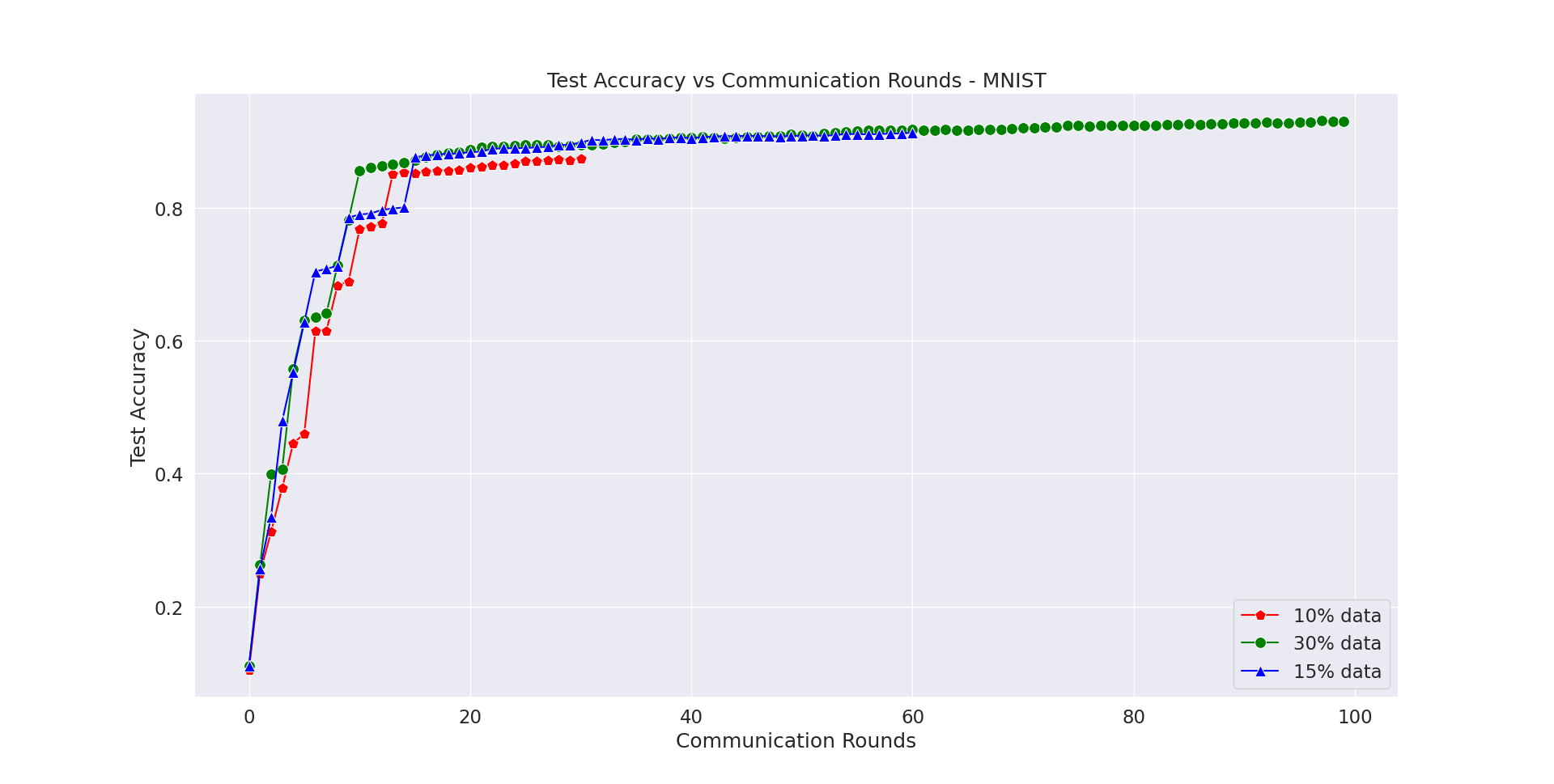}\hfill
\includegraphics[width=.5\textwidth]{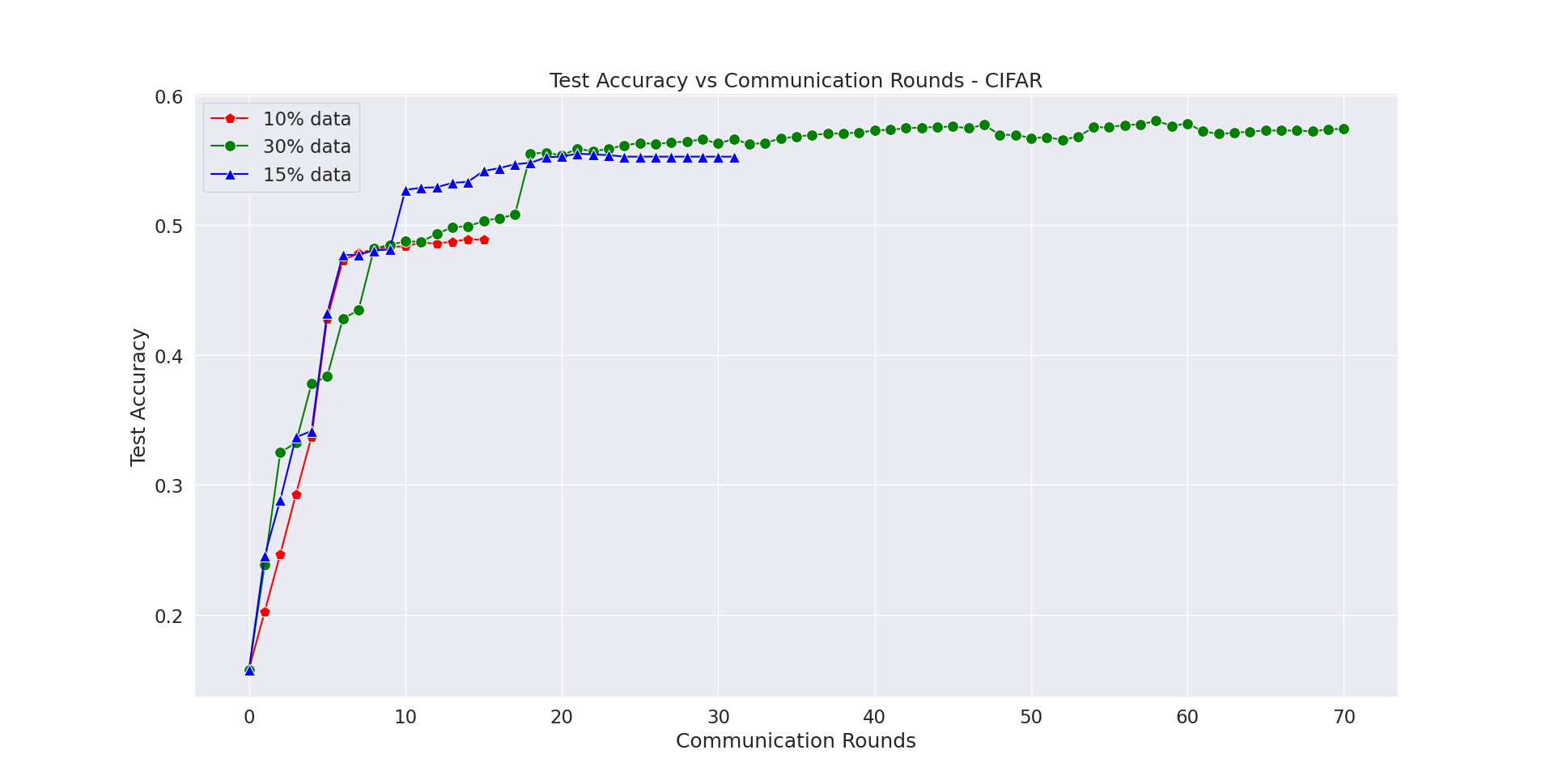}
\caption{Communication Rounds across Different Sample Size - Convergence analysis}
\label{fig:conv_analysis}
\end{figure}

\subsection{Computing likelihood objective using AIHT}
Here, we showcase how we utilised the Accelerated Iterative Hard Thresholding algorithm (A-IHT) for computing the likelihood.

\subsection{Medical Dataset experiment details}

Owing to the rise of Federated Learning based approaches in the medical setting due to privacy-preserving features, we chose to perform our experiments on 3 medical datasets in addition to our main experiments.

For our Federated Learning setup, we considered the setting where we have only 2 clients and one global server.

For each of the 3 datasets in the medical dataset setting, we consider each client has X-ray images of symptomatic type \textbf{A}/ type \textbf{B} and Normal images . We perform a classification task at each client.

\subsection{Basline Comparisons: Diversity based Submodular Optimization Functions}
\label{methods:submod}

For our second set of experiments, we chose different diversity based submodular optimization functions, specifically the following functions whose definition have been provided here

\begin{definition}
\textbf{Log-determinant Function} is a diversity-based submodular function. It is non-monotone in nature. Let $\mathbf{L}$ denote a positive semidefinite kernel matrix and $\mathbf{L_S}$ denote the subset of rows and columns indexed by set $\mathbf{S}$. Log-determinant function $f$ is specified as:
\begin{equation}
    f(\mathbf{S}) = \text{logdet}(\mathbf{L_S})
\end{equation}
\end{definition}

The log-det function models diversity and is closely related to a determinantal point process.

\begin{definition}
\textbf{Disparity Sum Function} characterizes diversity by considering the sum of distances between every pair of points in a subset $\mathbf{S}$. For any two points $i,j \in \mathbf{S}$, let $d_{ij}$ denote the distance between them.
\begin{equation}
    f(\mathbf{S})=\sum_{i, j \in \mathbf{S}}d_{i j}
\end{equation}
The aim is to select a subset $\mathbf{S}$ such that $f(\mathbf{S})$ is maximized. Disparity sum is not a submodular function.
\end{definition}

\begin{definition}
\textbf{Disparity Min Function} characterizes diversity by considering the minimum distance between any two non-similar points in a subset $\mathbf{S}$.
\begin{equation}
    f(\mathbf{S})=\min _{i, j \in \mathbf{S}, i \neq j}d_{i j}
\end{equation}
The aim is to select a subset $\mathbf{S}$ such that $f(\mathbf{S})$ is maximized. Disparity min is not a submodular function.
\end{definition}

For the above experiments we utilise the \textit{Submodlib library} \footnote{ \href{github.com/decile-team/submodlib}{ 
  Submodlib decile library}} for our implementation \cite{kaushal2022submodlib}.

\subsection{Experiment Configuration}

\subsubsection{MNIST experiment configuration}

For both $\methodprop$ and corresponding baseline $\baseline$, we use a fully connected DNN model with 3 layers [784,100,10] on MNIST dataset.

\textbf{Learning rate hyperparameters}:
As per \cite{zhang2022personalized}'s proposal i.e. $\baseline$ the learning rates for personalized (client model) and global model ($\eta_1, \eta_2$) are set to 0.001 since these choices result in the best setting for $\baseline$. To compare against the stable best hyperparameters of $\baseline$, we also fix the same for our proposal $\methodprop$.

\textbf{Personalization Hyperparameter}: The $\zeta$ parameter adjusts the degree of personalization in the case of clients. Again for a fair comparison against our baseline $\baseline$, we fix the $\zeta$ parameter for our proposal $\methodprop$ to the best setting given by the baseline. In \cite{zhang2022personalized} the authors tune $\zeta \in \{0.5, 1,5,10,20\}$ and find that $\zeta = 10$ results in the best setting. We, therefore, fix the personalization parameter $\zeta = 10$.

\subsubsection{Medical Datasets experiment configuration}
\label{medical datasets}

We discuss here the detailed configuration and models used for our further experiments.

Here we specifically consider the setting where we only have \textit{2} clients and a single global server. Each of the 2 clients are assigned with data from only 2 classes along with a shared class for classification purpose.

For example, client 1 has class $A$ and \textit{Normal} (shared class) images while client 2 has class $B$ and the remaining \textit{Normal} images.

\begin{figure}
    \centering
    \includegraphics[scale=0.25]{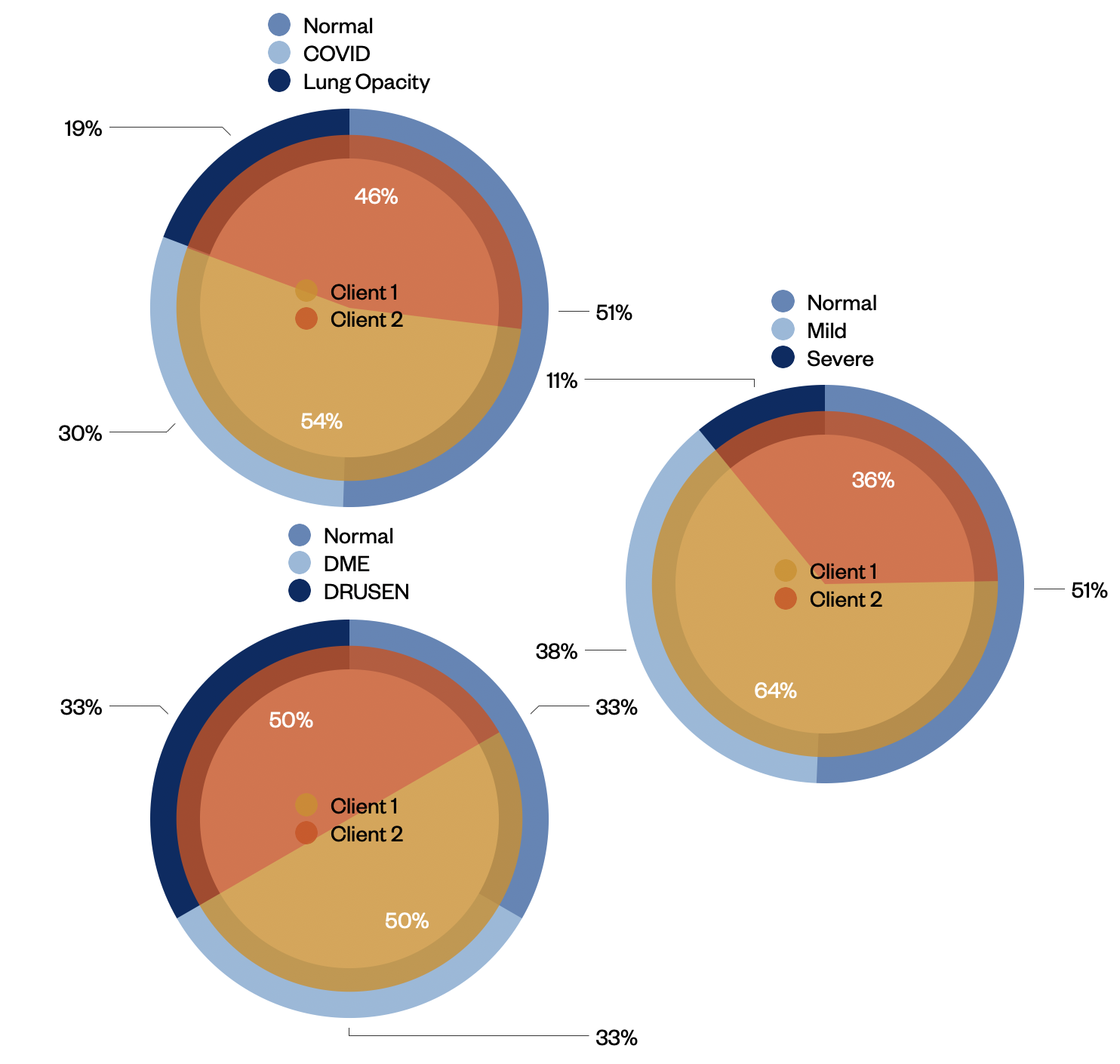}
    \caption{Data distribution for Medical Datasets}
    \label{fig:distribution}
\end{figure}

\textbf{COVID-19 Radiography Database}: Client 1 has COVID-19 x-ray images while client 2 has lung opacity x-ray images. Normal X-ray images are shared across both clients. Fig. \ref{fig:distribution} depicts the dataset distribution. For random subset selection, we randomly choose $\lambda = 0.1$ fraction of samples on the client side. For diversity-based subset selection, we first convert each of the 299$\times$299 images into a [512$\times$1] vector embeddings using a ResNet architecture. Diversity functions are then applied to these embeddings to retrieve a final subset of diverse and representative embeddings. Eventually, we decode back to the original space using the chosen representative indices.

 \textbf{APTOS 2019 Blindness Detection}: Unlike the COVID-19 radiography dataset, the APTOS dataset has 3 RGB channels and a higher resolution. We rescale the dimension of images to 299x299 for maintaining uniformity across all datasets. The same model configuration is followed as in the COVID-19 radiography dataset.

 \textbf{OCTMNIST}: 
    The OCTMnist dataset is a large dataset with single-channel images of a higher resolution. We have resized the images to 299$\times$299 resolution for our experiments. The Normal class has above 50,000 train images itself, with the other two classes having close to 10,000 train images. Due to this class imbalance, we have randomly selected 8,000 images from each class for our experiments. Post which we again use a ResNet architecture to reduce the feature dimensions, which we then feed into the $\methodprop$ pipeline.

\textbf{Baseline : Independent Learning}
In this scenario, each of the 2 clients solve the classification problem independently without any involvement of a server as opposed to federated learning. Thus there is no sharing of model weights to a common server as compared to the federated setting.

\textbf{Baseline : Independent Learning on other client's test data}
In this scenario, similar to the independent learning setup, we report the metrics for a particular client not only on its own test data but also on the other client's test data by training on the individual client's own training data.

\vspace{7pt}

For all the experiments for the medical dataset analysis across all the baselines, we report the classwise accuracy in Table \ref{table:accuracy-report}.

\begin{definition}
\textbf{Submodular Functions}
\label{Submodular_function}
are set functions which exhibit diminishing returns. Let $\mathbf{V}$ denote the ground-set of $n$ data points $\{x_1, x_2, \dots x_n\}$ where $x_i \in \mathbb{R}^d$. More formally, $\mathbf{V} = \{x_{i}\}_{i=1}^n$. Let $\mathbf{A} \subseteq \mathbf{B}$ where $\mathbf{A}, \mathbf{B} \subset \mathbf{V} \text{ and } v \in \mathbf{V}$. A submodular function $f : 2^V \mapsto \mathbb{R}$ satisfies the diminishing returns property as follows:
\begin{equation}
f(\mathbf{A} \cup v)-f(\mathbf{A}) \geq f(\mathbf{B} \cup v)  - f(\mathbf{B})
\end{equation}

\end{definition}

\subsection{Algorithm for Accelerated IHT} \label{algorithm:AIHT}
We first present the accelerated IHT algorithm as proposed in \cite{zhang2021bayesian} in Algorithm \ref{alg: AIHT}.

\begin{tabularx}{\textwidth}{l}
\hline${\textbf{ Algorithm 2}  \text { Accelerated IHT (A-IHT) for Bayesian Coreset Optimization }}$ \\
\hline \textbf{Input Objective} $f(w)=\|y-\Phi w\|_2^2 ;$ sparsity $k$ \\
 1: $t=0, z_0=0, w_0=0$ \\
 2: repeat \\
 3: $\quad \mathcal{Z}=\operatorname{supp}\left(z_t\right)$ \\
 4: $\quad \mathcal{S}=\operatorname{supp}\left(\Pi_{\mathcal{N}_k \backslash \mathcal{Z}}\left(\nabla f\left(z_t\right)\right)\right) \cup \mathcal{Z}$ where $|\mathcal{S}| \leq 3 k$ \\
 5: $\quad \widetilde{\nabla}_t=\left.\nabla f\left(z_t\right)\right|_{\mathcal{S}}$ \\
 6: $\quad \mu_t=\arg \min _\mu f\left(z_t-\mu \widetilde{\nabla}_t\right)=\frac{\left\|\tilde{\nabla}_t\right\|_2^2}{2\left\|\Phi \tilde{\nabla}_t\right\|_2^2}$ \\
 7: $\quad w_{t+1}=\Pi_{\mathcal{C}_k \cap \mathbb{R}_{+}^n}\left(z_t-\mu_t \nabla f\left(z_t\right)\right)$ \\
 8: $\quad \tau_{t+1}=\arg \min _\tau f\left(w_{t+1}+\tau\left(w_{t+1}-w_t\right)\right)$ \\
 $\quad=\frac{\left\langle y-\Phi w_{t+1}, \Phi\left(w_{t+1}-w_t\right)\right\rangle}{2\left\|\Phi\left(w_{t+1}-w_t\right)\right\|_2^2}$ \\
 9: $\quad z_{t+1}=w_{t+1}+\tau_{t+1}\left(w_{t+1}-w_t\right)$ \\
 10: $\quad t=t+1$ \\
 11: until Stop criteria met \\
 12: return $w_t$ \\
 \hline
 \label{alg: AIHT}
 \end{tabularx}

The algorithm \textbf{Accelerated IHT} above is proposed by 
\cite{zhang2021bayesian}. We share a high level view of the algorithm include some of the important features.

\textbf{Step Size Selection}
The authors propose that given
the quadratic objective of the coreset optimization, they perform exact line search to obtain the best step size per iteration. $\frac{\left\|\tilde{\nabla}_t\right\|_2^2}{2\left\|\Phi \tilde{\nabla}_t\right\|_2^2}$

\textbf{Momentum}
The authors propose adaptive momentum acceleration as is evident from line 8 of the pseudocode. At the end during the next update, Nesterov Accelerated Gradient is applied as shown in line 9.

\section{Code}
We share our code on GitHub at \href{https://github.com/prateekiiest/coresetFedML}{Link}
%We host our code on $\methodprop$ and related experiments on GitHub\footnote{ \href{github.com/fang84/coresetFedML}{  GitHub code}}

%% file: arxiv_ICLR_bayesian_coreset.bbl
\begin{thebibliography}{52}
\providecommand{\natexlab}[1]{#1}
\providecommand{\url}[1]{\texttt{#1}}
\expandafter\ifx\csname urlstyle\endcsname\relax
  \providecommand{\doi}[1]{doi: #1}\else
  \providecommand{\doi}{doi: \begingroup \urlstyle{rm}\Url}\fi

\bibitem[Bai et~al.(2020)Bai, Song, and Cheng]{bai2020efficient}
Jincheng Bai, Qifan Song, and Guang Cheng.
\newblock Efficient variational inference for sparse deep learning with theoretical guarantee.
\newblock \emph{Advances in Neural Information Processing Systems}, 33:\penalty0 466--476, 2020.

\bibitem[Balakrishnan et~al.(2021)Balakrishnan, Li, Zhou, Himayat, Smith, and Bilmes]{diverseblimes}
Ravikumar Balakrishnan, Tian Li, Tianyi Zhou, Nageen Himayat, Virginia Smith, and Jeff Bilmes.
\newblock Diverse client selection for federated learning via submodular maximization.
\newblock In \emph{International Conference on Learning Representations}, 2021.

\bibitem[Blundell et~al.(2015)Blundell, Cornebise, Kavukcuoglu, and Wierstra]{blundell2015weight}
Charles Blundell, Julien Cornebise, Koray Kavukcuoglu, and Daan Wierstra.
\newblock Weight uncertainty in neural network.
\newblock In \emph{International conference on machine learning}, pp.\  1613--1622. PMLR, 2015.

\bibitem[Campbell \& Beronov(2019)Campbell and Beronov]{campbell2019sparse}
Trevor Campbell and Boyan Beronov.
\newblock Sparse variational inference: Bayesian coresets from scratch.
\newblock \emph{Advances in Neural Information Processing Systems}, 32, 2019.

\bibitem[Campbell \& Broderick(2018)Campbell and Broderick]{campbell2018bayesian}
Trevor Campbell and Tamara Broderick.
\newblock Bayesian coreset construction via greedy iterative geodesic ascent.
\newblock In \emph{International Conference on Machine Learning}, pp.\  698--706. PMLR, 2018.

\bibitem[Coleman et~al.(2019)Coleman, Yeh, Mussmann, Mirzasoleiman, Bailis, Liang, Leskovec, and Zaharia]{deepcoreset1}
Cody Coleman, Christopher Yeh, Stephen Mussmann, Baharan Mirzasoleiman, Peter Bailis, Percy Liang, Jure Leskovec, and Matei Zaharia.
\newblock Selection via proxy: Efficient data selection for deep learning.
\newblock \emph{arXiv preprint arXiv:1906.11829}, 2019.

\bibitem[Collins et~al.(2021)Collins, Hassani, Mokhtari, and Shakkottai]{collins2021exploiting}
Liam Collins, Hamed Hassani, Aryan Mokhtari, and Sanjay Shakkottai.
\newblock Exploiting shared representations for personalized federated learning.
\newblock In \emph{International Conference on Machine Learning}, pp.\  2089--2099. PMLR, 2021.

\bibitem[Dwork et~al.(2006)Dwork, McSherry, Nissim, and Smith]{diff_privacy_1}
Cynthia Dwork, Frank McSherry, Kobbi Nissim, and Adam Smith.
\newblock Calibrating noise to sensitivity in private data analysis.
\newblock In \emph{Theory of cryptography conference}, pp.\  265--284. Springer, 2006.

\bibitem[Fallah et~al.(2020{\natexlab{a}})Fallah, Mokhtari, and Ozdaglar]{FL_MAML}
Alireza Fallah, Aryan Mokhtari, and Asuman Ozdaglar.
\newblock Personalized federated learning: A meta-learning approach, 2020{\natexlab{a}}.
\newblock URL \url{https://arxiv.org/abs/2002.07948}.

\bibitem[Fallah et~al.(2020{\natexlab{b}})Fallah, Mokhtari, and Ozdaglar]{NEURIPS2020_24389bfe}
Alireza Fallah, Aryan Mokhtari, and Asuman Ozdaglar.
\newblock Personalized federated learning with theoretical guarantees: A model-agnostic meta-learning approach.
\newblock In H.~Larochelle, M.~Ranzato, R.~Hadsell, M.F. Balcan, and H.~Lin (eds.), \emph{Advances in Neural Information Processing Systems}, volume~33, pp.\  3557--3568. Curran Associates, Inc., 2020{\natexlab{b}}.
\newblock URL \url{https://proceedings.neurips.cc/paper_files/paper/2020/file/24389bfe4fe2eba8bf9aa9203a44cdad-Paper.pdf}.

\bibitem[Ghorbani \& Zou(2019)Ghorbani and Zou]{datashapley}
Amirata Ghorbani and James Zou.
\newblock Data shapley: Equitable valuation of data for machine learning.
\newblock In \emph{International Conference on Machine Learning}, pp.\  2242--2251. PMLR, 2019.

\bibitem[Har-Peled \& Mazumdar(2004)Har-Peled and Mazumdar]{coresetkmeans}
Sariel Har-Peled and Soham Mazumdar.
\newblock On coresets for k-means and k-median clustering.
\newblock In \emph{Proceedings of the thirty-sixth annual ACM symposium on Theory of computing}, pp.\  291--300, 2004.

\bibitem[Huang et~al.(2022)Huang, Li, Sun, and Zhao]{huang2022coresets}
Lingxiao Huang, Zhize Li, Jialin Sun, and Haoyu Zhao.
\newblock Coresets for vertical federated learning: Regularized linear regression and $ k $-means clustering.
\newblock \emph{arXiv preprint arXiv:2210.14664}, 2022.

\bibitem[Jain et~al.(2021)Jain, Rush, Smith, Song, and Guha~Thakurta]{NEURIPS2021_f8580959}
Prateek Jain, John Rush, Adam Smith, Shuang Song, and Abhradeep Guha~Thakurta.
\newblock Differentially private model personalization.
\newblock In M.~Ranzato, A.~Beygelzimer, Y.~Dauphin, P.S. Liang, and J.~Wortman Vaughan (eds.), \emph{Advances in Neural Information Processing Systems}, volume~34, pp.\  29723--29735. Curran Associates, Inc., 2021.
\newblock URL \url{https://proceedings.neurips.cc/paper/2021/file/f8580959e35cb0934479bb007fb241c2-Paper.pdf}.

\bibitem[Jordan et~al.(1999)Jordan, Ghahramani, Jaakkola, and Saul]{jordan1999introduction}
Michael~I Jordan, Zoubin Ghahramani, Tommi~S Jaakkola, and Lawrence~K Saul.
\newblock An introduction to variational methods for graphical models.
\newblock \emph{Machine learning}, 37:\penalty0 183--233, 1999.

\bibitem[Kairouz et~al.(2021)Kairouz, Liu, and Steinke]{diff_privacy_2}
Peter Kairouz, Ziyu Liu, and Thomas Steinke.
\newblock The distributed discrete gaussian mechanism for federated learning with secure aggregation.
\newblock In \emph{International Conference on Machine Learning}, pp.\  5201--5212. PMLR, 2021.

\bibitem[Karimireddy et~al.(2020)Karimireddy, Kale, Mohri, Reddi, Stich, and Suresh]{scaffold}
Sai~Praneeth Karimireddy, Satyen Kale, Mehryar Mohri, Sashank Reddi, Sebastian Stich, and Ananda~Theertha Suresh.
\newblock Scaffold: Stochastic controlled averaging for federated learning.
\newblock In \emph{International Conference on Machine Learning}, pp.\  5132--5143. PMLR, 2020.

\bibitem[Katharopoulos \& Fleuret(2018)Katharopoulos and Fleuret]{datavalue}
Angelos Katharopoulos and Fran{\c{c}}ois Fleuret.
\newblock Not all samples are created equal: Deep learning with importance sampling.
\newblock In \emph{International conference on machine learning}, pp.\  2525--2534. PMLR, 2018.

\bibitem[Kaushal et~al.(2018)Kaushal, Sahoo, Doctor, Raju, Shetty, Singh, Iyer, and Ramakrishnan]{kaushal2018learning}
Vishal Kaushal, Anurag Sahoo, Khoshrav Doctor, Narasimha Raju, Suyash Shetty, Pankaj Singh, Rishabh Iyer, and Ganesh Ramakrishnan.
\newblock Learning from less data: Diversified subset selection and active learning in image classification tasks.
\newblock \emph{arXiv preprint arXiv:1805.11191}, 2018.

\bibitem[Kaushal et~al.(2019)Kaushal, Iyer, Kothawade, Mahadev, Doctor, and Ramakrishnan]{submodcoreset3}
Vishal Kaushal, Rishabh Iyer, Suraj Kothawade, Rohan Mahadev, Khoshrav Doctor, and Ganesh Ramakrishnan.
\newblock Learning from less data: A unified data subset selection and active learning framework for computer vision.
\newblock In \emph{2019 IEEE Winter Conference on Applications of Computer Vision (WACV)}, pp.\  1289--1299. IEEE, 2019.

\bibitem[Kaushal et~al.(2022)Kaushal, Ramakrishnan, and Iyer]{kaushal2022submodlib}
Vishal Kaushal, Ganesh Ramakrishnan, and Rishabh Iyer.
\newblock Submodlib: A submodular optimization library.
\newblock \emph{arXiv preprint arXiv:2202.10680}, 2022.

\bibitem[Killamsetty et~al.(2021)Killamsetty, Durga, Ramakrishnan, De, and Iyer]{gradmatch}
Krishnateja Killamsetty, S~Durga, Ganesh Ramakrishnan, Abir De, and Rishabh Iyer.
\newblock Grad-match: Gradient matching based data subset selection for efficient deep model training.
\newblock In \emph{International Conference on Machine Learning}, pp.\  5464--5474. PMLR, 2021.

\bibitem[Kirchhoff \& Bilmes(2014)Kirchhoff and Bilmes]{submodcoreset2}
Katrin Kirchhoff and Jeff Bilmes.
\newblock Submodularity for data selection in statistical machine translation.
\newblock In \emph{Proceedings of EMNLP}, pp.\  131--141, 2014.

\bibitem[Krizhevsky et~al.(2009)Krizhevsky, Hinton, et~al.]{krizhevsky2009learning}
Alex Krizhevsky, Geoffrey Hinton, et~al.
\newblock Learning multiple layers of features from tiny images.
\newblock 2009.

\bibitem[Lecun et~al.(1998)Lecun, Bottou, Bengio, and Haffner]{726791}
Y.~Lecun, L.~Bottou, Y.~Bengio, and P.~Haffner.
\newblock Gradient-based learning applied to document recognition.
\newblock \emph{Proceedings of the IEEE}, 86\penalty0 (11):\penalty0 2278--2324, 1998.
\newblock \doi{10.1109/5.726791}.

\bibitem[Li et~al.(2022)Li, Zhou, Tian, and Tao]{Li_2022_CVPR}
Shuangtong Li, Tianyi Zhou, Xinmei Tian, and Dacheng Tao.
\newblock Learning to collaborate in decentralized learning of personalized models.
\newblock In \emph{Proceedings of the IEEE/CVF Conference on Computer Vision and Pattern Recognition (CVPR)}, pp.\  9766--9775, June 2022.

\bibitem[Li et~al.(2020)Li, Sahu, Talwalkar, and Smith]{homomorphic_2}
Tian Li, Anit~Kumar Sahu, Ameet Talwalkar, and Virginia Smith.
\newblock Federated learning: Challenges, methods, and future directions.
\newblock \emph{IEEE Signal Processing Magazine}, 37\penalty0 (3):\penalty0 50--60, 2020.

\bibitem[Lu et~al.(2022)Lu, Wang, Chen, Qin, Xu, Dimitriadis, and Qin]{lu2022personalized}
Wang Lu, Jindong Wang, Yiqiang Chen, Xin Qin, Renjun Xu, Dimitrios Dimitriadis, and Tao Qin.
\newblock Personalized federated learning with adaptive batchnorm for healthcare.
\newblock \emph{IEEE Transactions on Big Data}, 2022.

\bibitem[Luo et~al.(2022)Luo, Xiao, and Song]{luo2022personalized}
Sichun Luo, Yuanzhang Xiao, and Linqi Song.
\newblock Personalized federated recommendation via joint representation learning, user clustering, and model adaptation.
\newblock In \emph{Proceedings of the 31st ACM International Conference on Information \& Knowledge Management}, pp.\  4289--4293, 2022.

\bibitem[Maheshwari et~al.(2020)Maheshwari, Chatterjee, Killamsetty, Ramakrishnan, and Iyer]{maheshwari2020semi}
Ayush Maheshwari, Oishik Chatterjee, Krishnateja Killamsetty, Ganesh Ramakrishnan, and Rishabh Iyer.
\newblock Semi-supervised data programming with subset selection.
\newblock \emph{arXiv preprint arXiv:2008.09887}, 2020.

\bibitem[Marfoq et~al.(2022)Marfoq, Neglia, Vidal, and Kameni]{pmlr-v162-marfoq22a}
Othmane Marfoq, Giovanni Neglia, Richard Vidal, and Laetitia Kameni.
\newblock Personalized federated learning through local memorization.
\newblock In Kamalika Chaudhuri, Stefanie Jegelka, Le~Song, Csaba Szepesvari, Gang Niu, and Sivan Sabato (eds.), \emph{Proceedings of the 39th International Conference on Machine Learning}, volume 162 of \emph{Proceedings of Machine Learning Research}, pp.\  15070--15092. PMLR, 17--23 Jul 2022.
\newblock URL \url{https://proceedings.mlr.press/v162/marfoq22a.html}.

\bibitem[McMahan et~al.(2017)McMahan, Moore, Ramage, Hampson, and Arcas]{pmlr-v54-mcmahan17a}
Brendan McMahan, Eider Moore, Daniel Ramage, Seth Hampson, and Blaise Aguera~y Arcas.
\newblock {Communication-Efficient Learning of Deep Networks from Decentralized Data}.
\newblock In Aarti Singh and Jerry Zhu (eds.), \emph{Proceedings of the 20th International Conference on Artificial Intelligence and Statistics}, volume~54 of \emph{Proceedings of Machine Learning Research}, pp.\  1273--1282. PMLR, 20--22 Apr 2017.
\newblock URL \url{https://proceedings.mlr.press/v54/mcmahan17a.html}.

\bibitem[Mirzasoleiman et~al.(2015)Mirzasoleiman, Karbasi, Badanidiyuru, and Krause]{submodcoreset4}
Baharan Mirzasoleiman, Amin Karbasi, Ashwinkumar Badanidiyuru, and Andreas Krause.
\newblock Distributed submodular cover: Succinctly summarizing massive data.
\newblock \emph{Advances in Neural Information Processing Systems}, 28, 2015.

\bibitem[Mirzasoleiman et~al.(2020{\natexlab{a}})Mirzasoleiman, Bilmes, and Leskovec]{craig}
Baharan Mirzasoleiman, Jeff Bilmes, and Jure Leskovec.
\newblock Coresets for data-efficient training of machine learning models.
\newblock In \emph{International Conference on Machine Learning}, pp.\  6950--6960. PMLR, 2020{\natexlab{a}}.

\bibitem[Mirzasoleiman et~al.(2020{\natexlab{b}})Mirzasoleiman, Cao, and Leskovec]{crust}
Baharan Mirzasoleiman, Kaidi Cao, and Jure Leskovec.
\newblock Coresets for robust training of deep neural networks against noisy labels.
\newblock In \emph{NeurIPS}, 2020{\natexlab{b}}.
\newblock URL \url{https://proceedings.neurips.cc/paper/2020/hash/8493eeaccb772c0878f99d60a0bd2bb3-Abstract.html}.

\bibitem[Nakada \& Imaizumi(2020)Nakada and Imaizumi]{nakada2020adaptive}
Ryumei Nakada and Masaaki Imaizumi.
\newblock Adaptive approximation and generalization of deep neural network with intrinsic dimensionality.
\newblock \emph{The Journal of Machine Learning Research}, 21\penalty0 (1):\penalty0 7018--7055, 2020.

\bibitem[Owen \& Daskin(1998)Owen and Daskin]{facilitylocation}
Susan~Hesse Owen and Mark~S Daskin.
\newblock Strategic facility location: A review.
\newblock \emph{European journal of operational research}, 111\penalty0 (3):\penalty0 423--447, 1998.

\bibitem[Pati et~al.(2018)Pati, Bhattacharya, and Yang]{pati2018statistical}
Debdeep Pati, Anirban Bhattacharya, and Yun Yang.
\newblock On statistical optimality of variational bayes.
\newblock In \emph{International Conference on Artificial Intelligence and Statistics}, pp.\  1579--1588. PMLR, 2018.

\bibitem[Polson \& Ro{\v{c}}kov{\'a}(2018)Polson and Ro{\v{c}}kov{\'a}]{polson2018posterior}
Nicholas~G Polson and Veronika Ro{\v{c}}kov{\'a}.
\newblock Posterior concentration for sparse deep learning.
\newblock \emph{Advances in Neural Information Processing Systems}, 31, 2018.

\bibitem[Segal et~al.(2017)Segal, Marcedone, Kreuter, Ramage, McMahan, Seth, Bonawitz, Patel, and Ivanov]{homomorphic_1}
Aaron Segal, Antonio Marcedone, Benjamin Kreuter, Daniel Ramage, H~Brendan McMahan, Karn Seth, KA~Bonawitz, Sarvar Patel, and Vladimir Ivanov.
\newblock Practical secure aggregation for privacy-preserving machine learning.
\newblock 2017.

\bibitem[Strubell et~al.(2019)Strubell, Ganesh, and McCallum]{strubell2019energy}
Emma Strubell, Ananya Ganesh, and Andrew McCallum.
\newblock Energy and policy considerations for deep learning in nlp.
\newblock \emph{arXiv preprint arXiv:1906.02243}, 2019.

\bibitem[T.~Dinh et~al.(2020)T.~Dinh, Tran, and Nguyen]{NEURIPS2020_f4f1f13c}
Canh T.~Dinh, Nguyen Tran, and Josh Nguyen.
\newblock Personalized federated learning with moreau envelopes.
\newblock In H.~Larochelle, M.~Ranzato, R.~Hadsell, M.F. Balcan, and H.~Lin (eds.), \emph{Advances in Neural Information Processing Systems}, volume~33, pp.\  21394--21405. Curran Associates, Inc., 2020.
\newblock URL \url{https://proceedings.neurips.cc/paper_files/paper/2020/file/f4f1f13c8289ac1b1ee0ff176b56fc60-Paper.pdf}.

\bibitem[T~Dinh et~al.(2020)T~Dinh, Tran, and Nguyen]{t2020personalized}
Canh T~Dinh, Nguyen Tran, and Josh Nguyen.
\newblock Personalized federated learning with moreau envelopes.
\newblock \emph{Advances in Neural Information Processing Systems}, 33:\penalty0 21394--21405, 2020.

\bibitem[Takezawa(2005)]{takezawa2005introduction}
Kunio Takezawa.
\newblock \emph{Introduction to nonparametric regression}.
\newblock John Wiley \& Sons, 2005.

\bibitem[Wei et~al.(2015)Wei, Iyer, and Bilmes]{submodcoreset1}
Kai Wei, Rishabh Iyer, and Jeff Bilmes.
\newblock Submodularity in data subset selection and active learning.
\newblock In \emph{International Conference on Machine Learning}, pp.\  1954--1963. PMLR, 2015.

\bibitem[Xiao et~al.(2017)Xiao, Rasul, and Vollgraf]{xiao2017fashionmnist}
Han Xiao, Kashif Rasul, and Roland Vollgraf.
\newblock Fashion-mnist: a novel image dataset for benchmarking machine learning algorithms, 2017.

\bibitem[Yoon et~al.(2020)Yoon, Arik, and Pfister]{datavaluation}
Jinsung Yoon, Sercan Arik, and Tomas Pfister.
\newblock Data valuation using reinforcement learning.
\newblock In \emph{International Conference on Machine Learning}, pp.\  10842--10851. PMLR, 2020.

\bibitem[Yurochkin et~al.(2019)Yurochkin, Agarwal, Ghosh, Greenewald, Hoang, and Khazaeni]{pmlr-v97-yurochkin19a}
Mikhail Yurochkin, Mayank Agarwal, Soumya Ghosh, Kristjan Greenewald, Nghia Hoang, and Yasaman Khazaeni.
\newblock {B}ayesian nonparametric federated learning of neural networks.
\newblock In Kamalika Chaudhuri and Ruslan Salakhutdinov (eds.), \emph{Proceedings of the 36th International Conference on Machine Learning}, volume~97 of \emph{Proceedings of Machine Learning Research}, pp.\  7252--7261. PMLR, 09--15 Jun 2019.
\newblock URL \url{https://proceedings.mlr.press/v97/yurochkin19a.html}.

\bibitem[Zhang et~al.(2023)Zhang, Long, Zhou, Yan, Zhang, Zhang, and Yang]{zhang2023dual}
Chunxu Zhang, Guodong Long, Tianyi Zhou, Peng Yan, Zijian Zhang, Chengqi Zhang, and Bo~Yang.
\newblock Dual personalization on federated recommendation.
\newblock \emph{arXiv preprint arXiv:2301.08143}, 2023.

\bibitem[Zhang et~al.(2021)Zhang, Khanna, Kyrillidis, and Koyejo]{zhang2021bayesian}
Jacky Zhang, Rajiv Khanna, Anastasios Kyrillidis, and Sanmi Koyejo.
\newblock Bayesian coresets: Revisiting the nonconvex optimization perspective.
\newblock In \emph{International Conference on Artificial Intelligence and Statistics}, pp.\  2782--2790. PMLR, 2021.

\bibitem[Zhang et~al.(2022{\natexlab{a}})Zhang, Yin, Chen, Huang, Nguyen, and Cui]{zhang2022pipattack}
Shijie Zhang, Hongzhi Yin, Tong Chen, Zi~Huang, Quoc Viet~Hung Nguyen, and Lizhen Cui.
\newblock Pipattack: Poisoning federated recommender systems for manipulating item promotion.
\newblock In \emph{Proceedings of the Fifteenth ACM International Conference on Web Search and Data Mining}, pp.\  1415--1423, 2022{\natexlab{a}}.

\bibitem[Zhang et~al.(2022{\natexlab{b}})Zhang, Li, Li, Guo, and Shao]{zhang2022personalized}
Xu~Zhang, Yinchuan Li, Wenpeng Li, Kaiyang Guo, and Yunfeng Shao.
\newblock Personalized federated learning via variational bayesian inference.
\newblock In \emph{International Conference on Machine Learning}, pp.\  26293--26310. PMLR, 2022{\natexlab{b}}.

\end{thebibliography}
